\documentclass{article} 
\usepackage{times,fullpage}

\usepackage{amsmath,amsfonts,bm}

\def\eqref#1{equation~\ref{#1}}

\def\1{\bm{1}}

\def\ra{{\textnormal{a}}}

\DeclareMathAlphabet{\mathsfit}{\encodingdefault}{\sfdefault}{m}{sl}
\SetMathAlphabet{\mathsfit}{bold}{\encodingdefault}{\sfdefault}{bx}{n}

\usepackage{hyperref}
\usepackage{url}

\usepackage{amsthm}
\usepackage{chemfig}
\usepackage{tikz-cd}
\usepackage{paralist}
\usepackage{algorithm}
\usepackage{algpseudocode}
\usepackage[algo2e]{algorithm2e} 

\usepackage{wrapfig}
\usepackage{caption}
\usepackage{subcaption}
\usepackage{adjustbox}

\newtheorem{theorem}{Theorem}[section]
\newtheorem{lemma}[theorem]{Lemma}
\newtheorem{corollary}[theorem]{Corollary}

\newtheorem{claim}[theorem]{Claim}
\theoremstyle{definition}

\newtheorem{definition}[theorem]{Definition}

\newtheorem{remark}[theorem]{Remark}
\newtheorem{assumption}[theorem]{Assumption}

\makeatletter
\definearrow{3}{-|>}{%
  \CF@arrow@shift@nodes{#3}%
  \expandafter\draw\expandafter[\CF@arrow@current@style,-CF]
    (\CF@arrow@start@node) -| (\CF@arrow@end@node)
    node[midway] (arrow@middle){} ;%
  \CF@arrow@display@label
    {#1}{0.5}{+}{\CF@arrow@start@node}
    {#2}{0.5}{-}{arrow@middle}%
}
\makeatother

\makeatletter
\definearrow1{s>}{%
    \ifx\@empty#1\@empty
    \expandafter\draw\expandafter[\CF@arrow@current@style,-CF](\CF@arrow@start@node)--(\CF@arrow@end@node);%
    \else
    \def\curvedarrow@style{shorten <=\CF@arrow@offset,shorten >=\CF@arrow@offset,}%
    \CF@expadd@tocs\curvedarrow@style\CF@arrow@current@style
    \expandafter\draw\expandafter[\curvedarrow@style,-CF](\CF@arrow@start@name)..controls#1..(\CF@arrow@end@name);
    \fi
}
\makeatother

\title{Provable Hierarchical Lifelong Learning with a Sketch-based Modular Architecture\thanks{Alphabetical ordering denotes equal contribution.}}

\author{
Zihao Deng\\
Washington U.\ St.\ Louis\\
zihao.deng@wustl.edu
\and
Zee Fryer\\
Google Research\\
zeef@google.com
\and
Brendan Juba\\
Washington U.\ St.\ Louis\\
bjuba@wustl.edu
\and
Rina Panigrahy \& Xin Wang\\
Google Research\\
\{rinap,wanxin\}@google.com
}

\newcommand{\la}{\textlangle}
\renewcommand{\ra}{\textrangle}

\begin{document}

\maketitle

\begin{abstract}
We propose a modular architecture for lifelong learning of hierarchically structured tasks. Specifically, we prove that our architecture is theoretically able to learn tasks that can be solved by functions that are learnable given access to functions for other, previously learned tasks as subroutines. We empirically show that some tasks that we can learn in this way are not learned by standard training methods in practice; indeed, prior work suggests that some such tasks cannot be learned by \emph{any} efficient method without the aid of the simpler tasks. We also consider methods for identifying the tasks automatically, without relying on explicitly given indicators.
\end{abstract}

\section{Introduction}\label{sec:intro}

How can complex concepts be learned? Human experience suggests that hierarchical structure is key: the complex concepts we use are no more than simple combinations of slightly less complex concepts that we have already learned, and so on. This intuition suggests that the learning of complex concepts is most tractably approached in a setting where multiple tasks are present, where it is possible to leverage what was learned from one task in another. Lifelong learning~\cite{silver2013lifelong,chen2018lifelong} captures such a setting: we are presented with a sequence of learning tasks and wish to understand how to (selectively) transfer what was learned on previous tasks to novel tasks. We seek a method that we can analyze and prove leverages what it learns on simple tasks to efficiently learn complex tasks; in particular, tasks that could not be learned without the help provided by learning the simple tasks first.

In this work, we propose an architecture for addressing such problems based on creating new modules to represent the various tasks. Indeed, other modular approaches to lifelong learning \cite{yoon2018lifelong,rusu2016progressive} have been proposed previously. But, these works did not consider what we view as the main advantage of such architectures: their suitability for theoretical analysis. We prove that our architecture is capable of efficiently learning complex tasks by utilizing 
the functions learned to solve previous tasks as components in an algorithm for the more complex task. In addition to our analysis proving that the complex tasks may be learned, we also demonstrate that such an approach can learn functions that standard training methods fail to learn in practice, including some that are believed not to be learnable, even in principle~\cite{klivans2009cryptographic}.
We also consider methods for automatically identifying whether a learning task posed to the agent matches a previously learned task or is a novel task. 

We note briefly that a few other works  considered lifelong learning from a theoretical perspective. An early approach by \cite{solomonoff1989system} did not seriously consider computational complexity aspects.
\cite{ruvolo2013ella} gave the first provable lifelong learning algorithm with such an analysis. But, the transfer of knowledge across tasks in their framework was limited to feature learning. In particular, they did not consider the kind of deep hierarchies of tasks that we seek to learn.

\subsection{Overview of the architecture}

The main technical novelty in our architecture over previous modular lifelong learners is that ours uses a particular type of internal data structure called \emph{sketches} \cite{ghazi2019recursive,panigrahy2019does}. All such data, including inputs from the environment, outputs from a module for another task, decisions such as choosing an action to take, or even descriptions of the modules themselves, are encoded as such sketches. Although sketches have a dense (vector) representation, they can also be interpreted as a kind of structured representation \cite[Theorem 9]{ghazi2019recursive} and are {\em recursive}; that is, they point to the previous modules/events that they arose from (Figure~\ref{fig:architecture-and-sketch}, right). However, in order to construct these sketches in \cite{ghazi2019recursive}, the structure of the network is assumed to be given. No algorithms for constructing such a hierarchical network of modules from training data were known. In this work we show a method to construct such a hierarchical network from training data. We provide an architecture and algorithms for learning from a stream of training inputs that produces such a network of modules over time.
This includes challenges of identifying each module, and 
discovering which other modules it depends on. 

Our architecture can be viewed as a variant of the Transformer architecture \cite{radford2021learning, shazeer2017outrageously}, particularly the Switch Transformer \cite{fedus2021switch} in conjunction with the idea of Neural Memory \cite{sketchmem2021}. Instead of having a single feedforward layer, the Switch Transformer has an array of feedforward layers that an input can be routed to at each layer. Neural Memory on the other hand is a large table of values, and one or a few locations of the memory can be accessed at each layer of a deep network. In a sense the Switch Transfomer can be viewed as having a memory of possible feedforward layers (although they use very few) to read from. It is viewing the memory as holding ``parts of a deep network'' as opposed to data, although this difference between program and data is artificial: for example, embedding table entries can be viewed as ``data'' but are also used to alter the computation of the rest of the network, and in this sense act as a ``program modifier''.

\begin{figure}
\begin{subfigure}{0.6\textwidth}
\includegraphics[trim = 20 75 20 100 , clip, width=\textwidth]{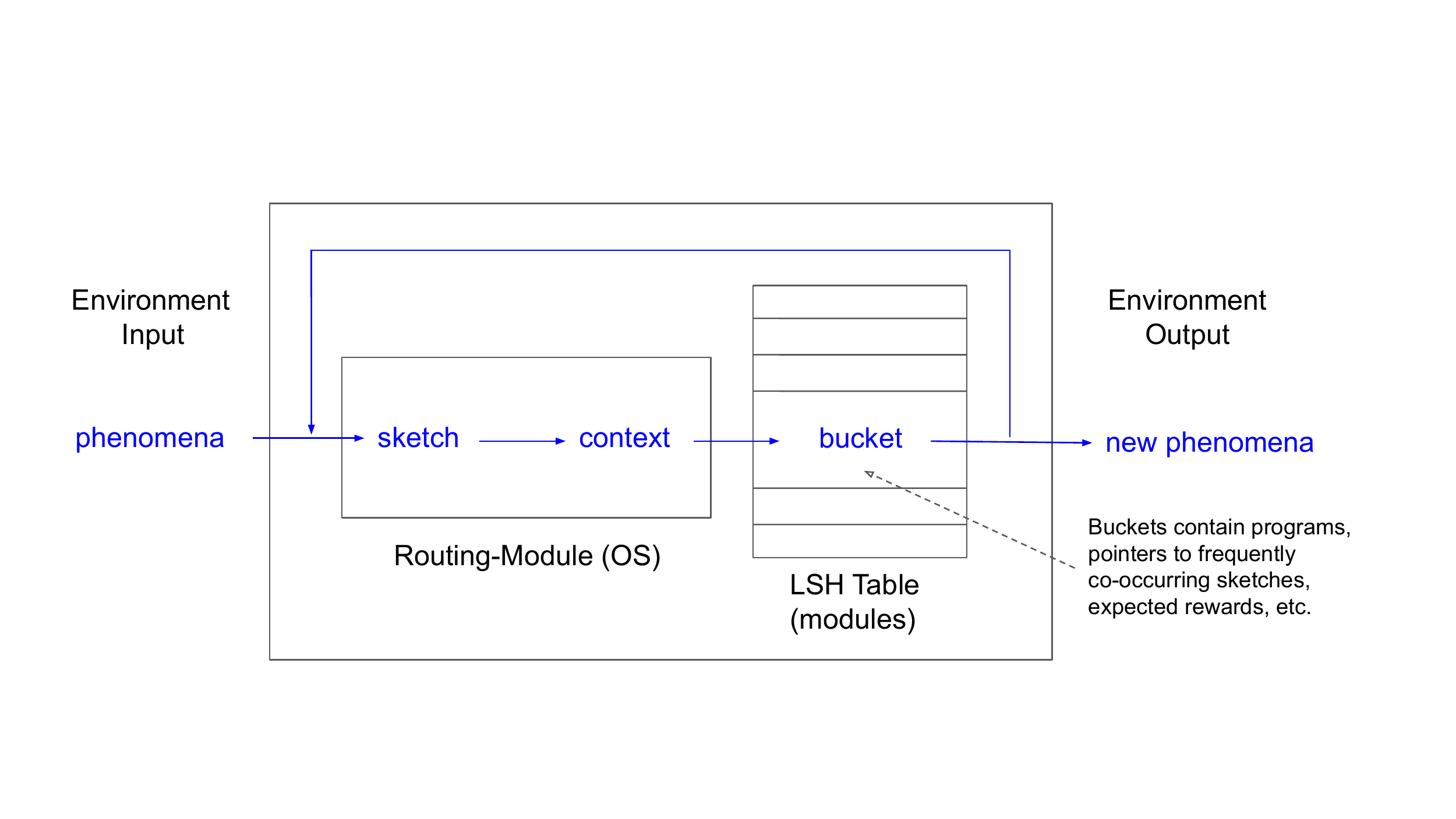}
\end{subfigure}
\begin{subfigure}{0.35\textwidth}
\begin{flushright}
\includegraphics[trim = 0 230 310 0, clip, width=\textwidth]{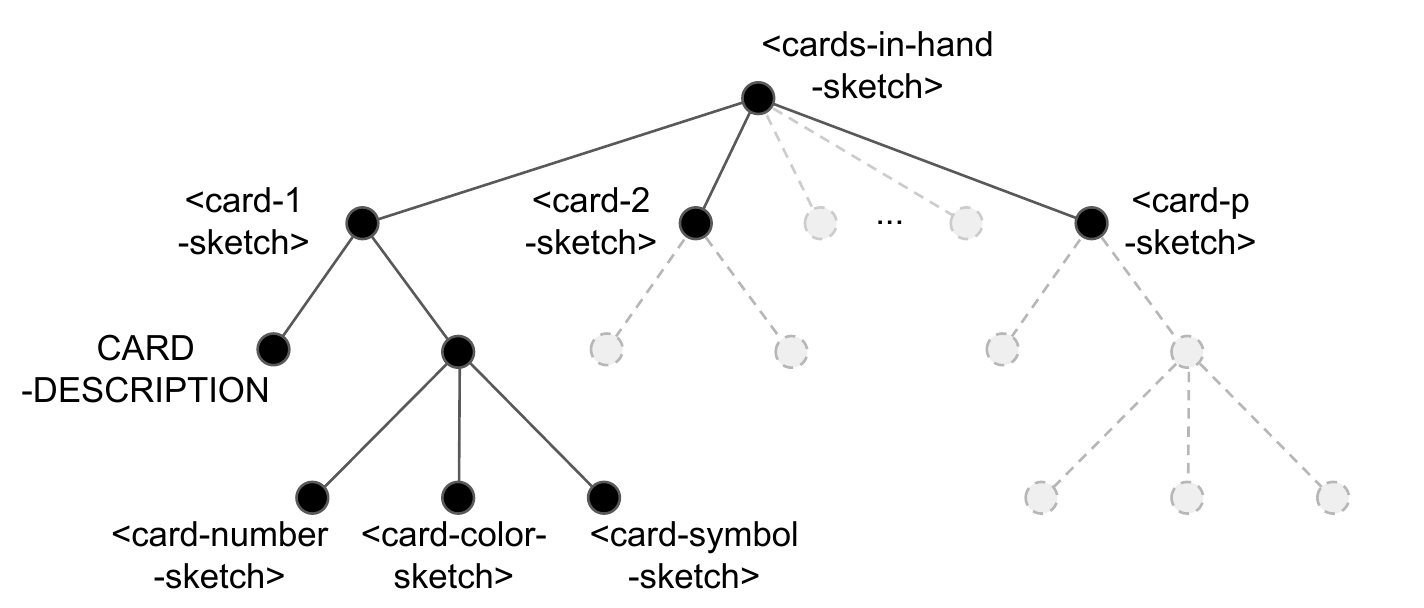}
\end{flushright}
\end{subfigure}
\caption{{\em (left)} The Routing-module (OS) routes sketches to the programs in the LSH table, which in turn produces sketches that are fed back to the OS in addition to sketches of inputs from the environment. The OS, while shown here as a distinct module, could itself be a module (program) in the LSH hash table. {\em (right)} Sketch of a hand of cards during a card game. The \la cards-in-hand-sketch \ra is a tuple of $p$ sub-sketches (one for each card), and each \la card-$i$-sketch \ra is itself a compound sketch: for example, \la card-1-sketch \ra consists of the CARD-DESCRIPTION label/type and three sub-sketches describing the card's number, color, and symbol.}
\label{fig:architecture-and-sketch}
\end{figure}

The key component of our architecture is a locality sensitive hash (LSH) table based memory (see \cite{sketchmem2021}) that holds  sketches of data (such as inputs) and {\em modules} or programs (think of an encoding of a small deep network) that handles such sketches (Figure~\ref{fig:architecture-and-sketch}, left).  The ``operating system'' of our architecture executes the basic loop of taking sketches (either from the environment or from internal modules) and routing/hashing them to the LSH table to execute the next module that processes these sketches. These modules produce new sketches that are fed back into the loop.

New modules (or concepts) are formed simply by instantiating a new hash bucket whenever a new frequently-occurring context arises, i.e. whenever several sketches hash to the same place; the context can be viewed as a function-pointer and the sketch can be viewed as arguments for a call to that function. Frequent subsets of sketches may be combined
to produce {\em compound} sketches. Finally we include pointers among sketches based on co-occurrence and co-reference in the sketches themselves. These pointers form a knowledge graph: for example if the inputs are images of pairs of people where the pairs are drawn from some latent social network, then assuming sufficient sampling of the network, this network will arise as a subgraph of the graph given by these pointers. The main loop allows these pointers to be dereferenced by passing them through the memory table, so they indeed serve the intended purpose.

The main idea of the architecture is that all events produce sketches, which can intuitively be thought of as the  ``mind-state'' of the system when that event occurs. The sketch-to-sketch similarity property (see below) combined with a similarity-preserving hash function ensures that similar sketches go to the same hash bucket (Appendix~\ref{appendix:lsh}); thus the hash table can be viewed as a content addressed memory. See Figure~\ref{fig:architecture-and-sketch} for an illustration of this.
We remark that the distances between embeddings of scene representations were used to automatically segment video into discrete events by \cite{franklin2020structured}, and obtained strong agreement with human annotators. The thresholded distance used to obtain the segmentation is analogous to our locality-sensitive hashes, which we use as context sketches.

A sketch can be viewed at different levels of granularity before using it to access the hash table; this becomes the {\em context} of the sketch. Each bucket contains a program that is executed when a sketch arises that indexes into that bucket. The program in turn produces outputs and new sketches that are routed back to the hash table. The system works in a continuous loop where sketches are coming in from the environment and also from previous iterations; the main structure of the loop is:

\adjustbox{scale=0.8,center}{
\begin{tikzcd}[every matrix/.append style={nodes={font=\small}}]
Phenomena \rar{input} & sketch \rar{f} & context \rar{h}  & bucket \rar  & program  \arrow[bend right=30, swap]{lll}{produces} \rar{{\tiny output}} & Phenomena
\end{tikzcd}
}

Thus external and internal inputs arrive as sketches that are converted into a coarser representation using a function $f$ (see Section~\ref{sec:architecture-principles} below) and then hashed to a bucket using a locality-sensitive hash function $h$. The program at that bucket is executed to produce an output-sketch that is fed back into the system and may also produce external outputs. This basic loop (described in Algorithm~\ref{algo:informal_execution-loop}) is executed by the routing module, which can be thought of as the operating system of the architecture.

\section{Sketches review}\label{sec:sketch-review}


Our architecture relies on the properties of the sketches introduced in \cite{ghazi2019recursive}. In this section we briefly describe the key properties of these sketches; the interested reader is referred to \cite{ghazi2019recursive, sketchmem2021} for details.

A \emph{sketch} is a compact representation of a possibly exponentially-wide ($d\times N$) matrix in which only a small number $k$ of the columns are nonzero, that supports efficient computation of inner products, and for which encoding and decoding are performed by linear transformations. For concreteness, we note that sketches may be computed by random projections to $\mathbb{R}^{d'}$ for $d'\geq kd\log N$; the Johnson-Lindenstrauss Lemma then guarantees that inner-products are preserved.

For our purposes, we suppose modules $M_1,\ldots,M_N$ produce vectors $x_1,\ldots,x_N\in\mathbb{R}^d$ as output, where only $k$ of the modules produce (nonzero) outputs. We view the sparse collection of module outputs as a small set of pairs of the form \{[$M_{i_1}$, $x_{i_1}$],\ldots,[$M_{i_k}$, $x_{i_k}$] \}: For example an input image has a sketch that can be thought of as a tuple [IMAGE, \la bit-map-sketch\ra ]. An output by an image recognition module that finds a person in the image can be represented as [PERSON, [\la person-sketch\ra , \la position-in-image-sketch\ra ]); here IMAGE and PERSON can be thought of as ``labels''. If the outputs of these modules are vector embeddings in the usual sense, then indeed the inner products measure the similarity of the objects represented by the output embeddings.

Observe that the constituent individual vectors $x_j$ in a sketch may themselves be sketches. E.g., \la person-sketch{\ra}  could be set of such pairs 
\{[NAME,\la name-sketch\ra ], [FACIAL-FEATURES,\la facial-sketch\ra ], [POSTURE,\la posture-sketch\ra ]\}, and an image consisting of multiple people could be mapped by our recognition module to a set \{\la person-1-sketch\ra , \la person-2-sketch\ra ,\ldots,\la person-$k$-sketch\ra \}.
Note if the if the tuple is very large, we will not be able to recover the sketch of each of its members but only get a ``average" or summary of all the sketches -- however if a member has high enough relative weight (see \cite[Section 3.3]{ghazi2019recursive}) it can be recovered.  Appendix~\ref{appendix:sketching} discusses how large objects can be stored as sketches hierarchically.

Indeed, following \cite{ghazi2019recursive}, the outputs of modules in our architecture will be tuples that, in addition to an ``output'' component, represent the input sketches which, in turn, represent the modules that produced those inputs, e.g., 
 \{[MODULE-ID,\la module-id\ra ], [OUTPUT-SKETCH,\la output-sketch\ra ],  [RECURSIVE-INPUT-SKETCH, \la recursive-input-sketch\ra ]\}. By recursively unpacking the input sketch, it is possible to reconstruct the history of computation that produced the sketch.

\subsection{Principles of the architecture}\label{sec:architecture-principles}

The following are the guiding principles behind the architecture. 


\begin{enumerate}
\item {\bf Sketches.}
All phenomena (inputs, outputs, commonly co-occurring events, etc) are represented as sketches.
There is a {\em function from sketch to context} $f: S\rightarrow C$ that gives a coarse grained version of the sketch. This is obtained by looking at the fields in the sketch $S$ that are high level labels and dropping fine details with high variance such as attribute values; it essentially extracts the ``high-level bits'' in the sketch $S$.
\item {\bf Hash table indexed by context that is robust to noise.} {\em (more details in Appendix~\ref{sec:appendixarchitectureprinciples})}
The hash function $h: C\rightarrow \text{[hash-bucket]}$ is ``locality sensitive'' in the sense that similar contexts are hashed to the same bucket with high probability.
Each hash bucket may contain a trainable program $P$, and summary statistics as described in Figure~\ref{fig:hashbucket}. We don't start to train $P$ until the hash bucket has been visited a sufficient number of times.
(Note: A program may not have to be an explicit interpretable program but could just be an ``embedding'' that represents (or modifies) a neural network.)
\item {\bf Routing-module.} {\em (can be implemented by Alg.\ref{algo:fdecisiontree} )}
Given a set of sketches from the previous iteration, the routing module identifies the top ones, applies the context function $f$ followed by the hash function $h$ to route them to their corresponding buckets. 
Before feeding these to $f$, the routing module picks a subset of sketches and combines them into compound sketches. The routing module makes all subset-choosing decisions 
\end{enumerate}


In addition, we can also keep  associations  of frequently co-occuring sketch contexts as edges across buckets forming knowledge graph. Please see Appendix~\ref{appendix:architecture} and \ref{app:knowledgegraph} for details. 


\begin{algorithm}[t]
\SetStartEndCondition{ }{}{}%
\SetKwProg{Fn}{def}{\string:}{}
\SetKwFunction{Range}{range}
\SetKw{KwTo}{in}\SetKwFor{For}{for}{\string:}{}%
\SetKwIF{If}{ElseIf}{Else}{if}{:}{elif}{else:}{}%
\SetKwFor{While}{while}{:}{fintq}%
\DontPrintSemicolon%
\AlgoDontDisplayBlockMarkers\SetAlgoNoEnd\LinesNumbered%
\SetAlgoSkip{}

\KwIn{input sketch $T$ (this sketch may contain a desired output for training)} 
{\tt current-sketches} $\leftarrow \{T\}$ \;

\While{{\tt current-sketches} is not empty}{
{\tt current-programs} $\leftarrow$ $\emptyset$\;

\ForEach{sketch $S$ in {\tt current-sketches}\, }{
extract context $C = f(S)$ \;

update access-frequency-count of bucket $h(C)$ \;

\uIf{bucket $h(C)$ has a program $P$}{append $(S,P)$ to {\tt current-programs}
}
\uElse{
\uIf{bucket $h(C)$ is frequently accessed}
    {initialize program at $h(C)$ with some random program
    and mark it for training.\;
    
    Fetch programs $P_i$ (possibly by some similarity criterion), append those $(S,P_i)$ to {\tt current-programs}}}
    }

Routing module chooses some subset of {\tt current-programs}, runs each program on its associated sketch, appends outputs to {\tt current-sketches}\; \label{Routing:choose_current_programs}

Append sketches on outgoing edges of accessed buckets to {\tt current-sketches}\; 

\uIf{any of the programs are marked for training}
    {routing module picks one or some of them and trains them, and may choose to stop execution loop}
    
\uIf{any of the sketches is of (a special) type OUTPUT sketch}
    {routing module picks one such, outputs that sketch or performs that action, and may choose to stop execution loop}

\uIf{any of the sketches is of type REWARD sketch (say for correct prediction or action)}
    { 
    updates the reward for this bucket and propagates those rewards to prior buckets}    
    
Routing module 
picks  $k$ combinations of sketches in  {\tt current-sketches}, and combine them into compound sketches:  
$S_1, \dots, S_k$ 
(may produce 0 sketches)\; \label{Routing:combining_current_sketches}

{\tt current-sketches} $\leftarrow \{S_1, \dots, S_k\}$\;
}
\caption{Informal presentation of the main execution loop 
}
\label{algo:informal_execution-loop}
\end{algorithm}

Thus new modules (or concepts) are formed simply by a new frequently occurring context (see earlier papers on how sketches are stored in LSH based sketch memory). Since sketches are fed to the programs indexed by their context, the context can be viewed as a function-pointer and the sketch can be viewed as arguments for a call to that function; multiple arguments can be passed by using a compound sketch. Programs that call other modules can be represented as a computation DAG over modules at the nodes 
.

\section{Independent tasks and 
architecture v0}\label{sec:v0}


Our learning problem follows a formulation similar to \cite{ruvolo2013ella}. In a lifelong learning system, we are facing a sequence of supervised learning tasks: ${\cal{Z}}^{(1)}, \dots, {\cal{Z}}^{(T_{max})}$. In contrast to \cite{ruvolo2013ella}, at each step we will generally obtain a single input (in the form of sketch $(s_t, {\mathbf{x}}^{(t)},\mathbf{y}^{(t)})$) that contains DATA 
${\bf{x}}^{(t)} \in {{\cal{X}}^{(t)}}$, TARGET ${\bf{y}}^{(t)} \in {{\cal{Y}}^{(t)}}$ and task descriptor sketch (vector) $s_t$, 
where the $t$th task is given by a hidden function ${\hat{\phi}}^{(t)}: {{\cal{X}}^{(t)}} \to {{\cal{Y}}^{(t)}}$ that we want to learn, 
and ${\hat{\phi}}^{(t)}({\bf{x}}^{(t)})={y}^{(t)}$. We assume that that the tasks are uniformly distributed, and the distributions over task data are stationary: i.e., at each step, the task is sampled uniformly at random, and for the sampled task $t$, the data is sampled independently 
from a fixed distribution 
on ${{\cal{X}}^{(t)}}$ for $t$.
In this setting, we assume the task functions are all members of a common, known class of functions  $\mathcal{M}$ for which there exists an efficient learning algorithm $\mathcal{A}_{\mathcal{M}}$, i.e., $\mathcal{A}_{\mathcal{M}}$ satisfies a standard PAC-learning guarantee: when provided with a sufficiently large number of training examples $M$, with probability $1-\delta$ $\mathcal{A}_{\mathcal{M}}$ returns a function that agrees with the task function with probability at least $1-\epsilon$ on the task distribution. For example, SGD learns  a certain class of neural networks $\mathcal{M}$ with a small constant depth. Indeed, we stress that this setting does not require transfer across tasks. 
Our architectures are instantiated by a choice of hash function $h$ and a context function $f$. 
Architecture v0 uses a very simple context function $f$: it projects the sketch down to the task descriptor $t$, dropping the DATA ${\mathbf{x}}^{(t)}$ and TARGET ${y}^{(t)}$ components. (Other combinatorial decisions in the routing module are NO-OPs.)
Each time we receive an input learning sample, we will call Alg.~\ref{algo:informal_execution-loop} (input is a single sketch).

\begin{claim}\label{claim:independent-tasks}
Given an error rate $\epsilon>0$ and confidence parameter $\delta>0$ and $N$ independent tasks, each of which require at most $M=M(\epsilon,\delta)$ examples to learn to accuracy $1-\epsilon$ with probability $1-\delta$, and training data as described in above, with probability $1-(N+1)\delta$, Architecture v0 learns to perform all $N$ tasks to accuracy at least $1-\epsilon$ in $O(M N\log \frac{N}{\delta})$ steps.
\end{claim}

\section{Hierarchical lifelong learning and architecture v1}
\label{sec:hll}


We follow a similar problem formulation as in Sec.\ref{sec:v0}, but in this case a task can depend on other tasks. We assume that the structure of dependencies can be described by a degree-$d$ directed acyclic graph (DAG), in which the nodes correspond to tasks. 
Each task $t$ depends on at most $d$ other tasks $t'_1, \dots, t'_d$, indicated by the nodes in the DAG with edges to its node, and the task is to compute the corresponding function ${\hat{\phi}}^{(t)} = {\phi}^{(t)} ({\hat{\phi}}^{(t'_1)}, \dots, {\hat{\phi}}^{(t'_d)})$ where $ {\phi}^{(t)} \in \mathcal{M}$. 
If $t'_1, \dots, t'_d$ are sources in the DAG (no incoming edges) then ${\hat{\phi}}^{(t'_i)} \in \mathcal{M}$. 
We assume that all tasks share a common input distribution. 
We will call the functions from $\mathcal{M}$ \emph{atomic modules}, since they are the building blocks of this hierarchy. We will call functions that call other functions in the DAG, such as ${\hat{\phi}}^{(t)}$ above, a \emph{compound module}. As before, we assume $\mathcal{M}$ is a learnable function class. However, ${\hat{\phi}}^{(t)}$ might not belong to a learnable function class due to its higher complexity.
Here, we will assume moreover that the algorithm $\mathcal{A}_{\mathcal{M}}$ for learning $\mathcal{M}$ is robust to label noise. Concretely, we will assume that if an $\epsilon$-fraction of the labels are corrupted by an adversary, then $\mathcal{A}_{\mathcal{M}}$ produces an $O(\epsilon)$-accurate function. We note that methods are known to provide SGD with such robustness for strongly convex loss functions, even if the features are corrupted during training~\cite{diakonikolas2019sever} (see also, e.g., \cite{li2020gradient,shah2020choosing}).
In this setting, we assume that the tasks are again sampled uniformly at random, and that the data is sampled independently from a common, fixed distribution 
for all tasks.

As with the architecture  v0, 
v1 uses any locality sensitive hash function $h$ and a context function $f$ that projects the input sketch down to the task descriptor, discarding other components. The primary modifications are that
\begin{compactenum}
\item v1 tracks whether tasks are ``well-trained,'' freezing their parameters when they reach a given accuracy for their level of the hierarchy, and
\item until a ``well-trained'' model is found, we train candidate models for the task in parallel that use the outputs of each possible subset of up to $d$ well-trained modules as auxiliary inputs.
\end{compactenum} 
We will maintain a global task level $L$, initially $0$. We define the target accuracy for level-$L$ tasks to be $\epsilon_L=(2dC)^L\epsilon$, where $C$ is the constant under the big-O for the guarantee provided by our robust training method; we let $M_L$ denote the sample complexity of robustly learning members of our class $\mathcal{M}$ to accuracy $1-Cd\epsilon_{L-1}$ with confidence $1-\delta$ when a $1-d\epsilon_{L-1}$-accurate model exists. We check if any tasks became well-trained in level $L-1$, and if so, for all tasks that are not yet well-trained, we initialize models for all combinations of up to $d-1$ other well-trained tasks for each such new task. Each model is of the form $\phi( \hat{\phi}_{i_1}(x),..., \hat{\phi}_{i_k}(x) )$, where $i_1,\ldots i_k$ ($k\leq d$) is the corresponding subset of well-trained tasks such that at least one has level $L-1$.
On each iteration, the arriving example is hashed to a bucket for task $t$. We track the number of examples that have arrived for $t$ thus far at this level. 
For the first $M'$ examples that arrive in a bucket, we pass the example to the training algorithms for each model for this task, which for example completes another step of SGD. 
Once $M_L$ examples have arrived, we count the fraction of the next  $O(\frac{d}{\epsilon_L}\log\frac{N}{\delta})$  examples that are classified correctly by each of the models.  We thus check if its empirical accuracy is guaranteed to be at least $1-\epsilon_L$ with high probability. If the empirical accuracy is sufficiently high, we mark the task as well-trained and use this model for the task, discarding the rest of the candidates.  Once all of the tasks are well-trained or have obtained $M_L +O(\frac{d}{\epsilon_L}\log\frac{N}{\delta})$ examples since the global level incremented to $L$, we increment the global level to $L+1$.

\begin{lemma}
Suppose at each step, a task $t$ is chosen uniformly random from the set of tasks $\{t_1, \dots, t_N\}$ in a DAG of height $\ell$, along with one random sample $(x, y)$ where ${\hat{\phi}}^{(t)}(x) = y$. Then after 
$\ell M N \ln(1/\delta)$ steps all the tasks will be well-trained 
{ (training error rate $\leq \epsilon_L$ for each module at level $L$) }
 w.h.p. We will call SGD $O( \ell M N^{(1+d)} \ln(1/\delta))$ times during the training. Here, $M$ is
 the upper bound of all $M_L$. 
\end{lemma}

In the above discussion we argued at an algorithmic level and ignored the specific architecture details of which buckets the ${N \choose d}$ candidate modules are trained and how eventually a single compound module gets programmed in the bucket $h(s_t)$. See Appendix~\ref{sec:appendixarchitecturev1} for those details.

\section{architecture v2: Tasks without precise explicit descriptions
}\label{section:v1_5}

We follow a similar problem formulation as in Sec.~\ref{sec:hll}.  
However now clear task descriptors may not be provided externally, but may implicitly depend on the output of a previous module. 
(detailed examples in Appendix.\ref{sec:appendixarchitecturev2} ).
The following definitions and assumption differ from Sec.~\ref{sec:hll}.

\begin{definition}[Tasks]
Let $\cal{U}$ be a space of all (potentially recursive) sketches that include the input and output of all modules. 
($\cal{U}$ can be polymorphic, that is, it can contain multiple different data types).  
Each task $t_i$ is a mapping  $:
\mathcal{U} \to \cal{U}$. 
The input distribution of $t_i$, $\mathcal{D}_{i}$, is supported on $\cal{U}$. 
\end{definition}

\begin{definition}[Latent dependency DAG]
The latent {\em dependency} DAG is a DAG with nodes corresponding to tasks $t_1,..,t_N$ and edges indicating dependencies. Each task at an internal node depends on at most $d$ other tasks ($d$  may not be known to the learner, but is a small quantity). 
\end{definition}
\begin{definition}[Latent circuit]
Given a dependency DAG,
for each task $t_i$ there is a latent circuit with gates (nodes) corresponding to the tasks $t'_i$ that it depends on. 
In this circuit for $t_i$, there are (potentially) multiple sinks (nodes with no outgoing edges). 
The output of these sinks will be the inputs to  some atomic module, which gives the output of $t_i$. 
There are multiple atomic internal modules for each $t_i$ and the circuit routes each example to one of these modules. 
Each $t_i$ is ``vague'' in the sense that there are multiple modules that can cater to an example of this task.
\end{definition}

\begin{definition}[Hidden task description / Context]
Given the circuit of each $t_i$,
there is a fixed (unknown) subset of the outputs of the circuit that give a context value that uniquely identifies $t_i$.
There exists a bound $g_i$ on the number of context values for $t_i$.
There is one atomic module for each context. 
We let $G$ be an integer such that $\sum_ig_i\leq G$. 
\end{definition}

\begin{assumption}[No distribution shift]

For a latent dependency DAG and circuit for task $t_i$,
suppose $t_j$ is one of the nodes in the circuit of task $t_i$, and
let $ \mathbf{x}  \sim \mathcal{D}_{i}$ be the input to $t_i$.
For each $\mathbf{x}_{j}$ computed as an input to $t_j$ when the circuit is evaluated on $\bf{x}$, 
we assume $\mathbf{x}_{j}$ belongs to $\mathcal{D}_{j}$.  


\end{assumption}

Given the problem set-up above, we present our main result for this section:

\begin{theorem}[Learning DAG using v2]\label{Architecture:v15_main_theorem}
Given a latent dependency DAG of tasks over $N$ nodes and height $\ell$, and a circuit per internal node in the DAG, 
there exists an architecture v2 that learns all these tasks (up to error rate $\epsilon_L$ as defined in Sec.\ref{sec:hll}) with at most $O(\ell G M 2^{O(d^2 +  d \log(N/d))})$ steps. 
\end{theorem}

\subsection{Context function as a decision tree}\label{sec:v2.decision-tree}
In architecture v2 we use a more complicated context function $f$ to extract the stable context for each task. 
The context function can be implemented as a modular decision tree where each node is a separate module.
We are given a compound sketch $[S_1,..,S_k]$ where we assume the sketches are ordered by importance (e.g., based on
frequency: if there are $m$ hash buckets we will only track contexts that appear at least $O(1/m)$ fraction of the time, while others get ``timed out`` --  we assume $m \ge G 2^{O(d^2 +  d \log(N/d))}$). We  apply $f$ recursively and then over $f(S_1),..,f(S_k)$ from left to right in a decision tree where each branch either keeps or drops each item and stops or continues based on what obtains the highest rewards, tracked at each node (subtree) of the decision tree.
Thus the context function can be implemented as a recursive call to a decision tree { f}([$S_1,..,S_k$]) = {\bf DecisionTree}([f($S_1$),..,f($S_k$)]) (see
Alg.2
)---each node of the decision tree will be implemented in a  separate module (hash bucket). 

\begin{algorithm}

\SetStartEndCondition{ }{}{}%
\SetKwProg{Fn}{def}{\string:}{}
\SetKwFunction{Range}{range}
\SetKw{KwTo}{in}\SetKwFor{For}{for}{\string:}{}%
\SetKwIF{If}{ElseIf}{Else}{if}{:}{elif}{else:}{}%
\SetKwFor{While}{while}{:}{fintq}%
\DontPrintSemicolon%
\AlgoDontDisplayBlockMarkers\SetAlgoNoEnd\LinesNumbered%
\SetAlgoSkip{}

{\bf DecisionTree}([$C_1,..C_k$]) = {\bf TreeWalk}([], [$C_1,..C_k$]) \;

\BlankLine
{\bf TreeWalk}($l$, [$C_i,..,C_k$]) = branch-based-on-argmax ( \;

\Indp    reward(h([TREE-WALK, $l$.append($C_i$)])]) :{\bf TreeWalk}($l$.append($C_i$), [$C_{i+1},..,C_k$]) 
    /* keep $C_i$ in $l$ and proceed to next field */\; 
    
    reward(h([TREE-WALK, $l$])]): {\bf TreeWalk}(
    $l$, [$C_{i+1},..,C_k$])
    /* drop $C_i$ and proceed to next field */\;
    
    reward(h([TREE-WALK, $l$.append(END-WALK-SYMBOL)]) ) : 
    {\bf TreeWalk}($l$, [])
    /* exit the walk and output $l$ */\;
    
\Indm)
/* $l$ is the subset of fields from the sketch from the prefix processed so far, $C_i,..,C_k$ is the remaining part of the sketch. 
$l$ is used as the context for this current decision tree node and [$C_1,..,C_k$] is the input sketch. 
Each distinct value of $l$ is a separate decision tree node */
\caption{
Decision Tree
}
\label{algo:fdecisiontree}
\end{algorithm}

The branch statement is branching to a one of the three buckets:
h([TREE-WALK, l.append($C_i$)]), h([TREE-WALK, $l$])], or h([TREE-WALK, $l$.append(END-WALK-SYMBOL)]) based on the rewards%
; each bucket continues the decision tree walk with the rest of the entries in the list of contexts.
Note that during training the branch will be a probabilistic softmax rather than a deterministic argmax, with a temperature parameter $T$ that controls the exploration of the branches and decreases eventually to near $0$; thus the probability of each branch is proportional to $e^{- R_{branch}/T }$, where $R_{branch}$ is the reward of the branch. Initially all rewards are 0 and so all branching probabilities are all equal to $1/3$ (but there could be some other priors). Over time as the temperature is lowered, the probability concentrates gradually on the bucket with maximum reward. See Appendix \ref{sec:appendixarchitecturev2} for full details.


\begin{claim}\label{claim:decision-tree-appendix}
If $p$ is the initial probability of taking the optimal reward path to the leaf in the DecisionTree algorithm above, there is a schedule for the temperature in Algorithm~\ref{algo:fdecisiontree}, so that in $O(1/p \log 1/\delta)$ tree walk steps the modules at the nodes of the tree will converge so that the decision tree achieves optimal rewards with high probability $1-\delta$.
\end{claim}
\begin{proof}
We will keep a very high initial temperature $T$ (say $\infty$) for $O((1/p)\log (1/\delta))$ tree walk steps and then suddenly freeze it to near zero (which converts the softmax to a max) after these steps are finished. In these initial steps with high probability $1-\delta$ the optimal path to the leaf will have been visited at least $1/\delta$ times. Since each node is tracking the optimal rewards in its subtree, the recorded best path from root will have tracked at least this optimal reward. 
\end{proof}


\subsection{Incrementally learning a new Node (implicit task)}
In this subsection we provide an induction proof sketch for Theorem \ref{Architecture:v15_main_theorem}.  
In the previous subsection, we saw how the context function can be implemented as a probabilistic decision tree. 
Other functions of the routing module that involve making 
subset-choosing decisions
(such as Lines \ref{Routing:choose_current_programs} \& \ref{Routing:combining_current_sketches} in Alg.~\ref{algo:informal_execution-loop}):  for example, selecting a subset of $d$ pre-existing modules as children of a new task in v1 can be done using a separate decision tree 
(e.g.\ Alg.~2) where 
one needs to select a subset of at most $d$. This again becomes very similar to the operation of the context function: we just need to input all matured modules of the previous layer to Alg.\ 2 and find the $\leq d$ child modules. 
In architecture v2 any subset-choosing decision in our architecture can be done by using Alg.\ 2. 

The learning algorithm follows the framework of Alg.\ref{algo:informal_execution-loop}. 
The circuit routing is also done by Alg.~\ref{algo:fdecisiontree}: we feed all the $O( {N \choose d} 3^{d \choose 2} )$ candidate edges of the circuit to Alg.~\ref{algo:fdecisiontree}, which finds the correct subset. 
The inductive guarantee that lower-level tasks are well-trained comes from the bottom-up online algorithm of v1. 
Modules 
are marked as mature based on performance, and new modules are only built on top of mature previous nodes.
The probability of picking the right sequence of decisions for perform the new task is $p = 1/  2^{O(d^2 +  d \log(N/d))}$ (including which identifying which previous possibly implicit tasks it depends on and wiring them correctly with the right contexts) and it takes $M$ examples to train the task, then the task can be learned in $O(M 2^{O(d^2 +  d \log(N/d))}  )$ steps per atomic module (see Appendix \ref{sec:appendixarchitecturev2} for full details).

\section{Experiments}
We empirically examine two tasks for which learning benefits from using a modular architecture in this section. We compare an ``end to end'' learning approach to a modular learning approach which explores a DAG of previously learned tasks probabilistically.

\subsection{Learning intersections of halfspaces}
Learning intersections of halfspaces has been studied extensively, see for example \cite{klivans2009cryptographic}. We first describe the experiment setting. Let $K$ be the number of hyperplanes, $D$ feature space dimension, we generate the following data:  hyperplane coefficients $w_k \in \mathbb{R}^{D}$, $k=1, 2, ..., K$ whose components are independent and follow standard normal distribution; 2.feature $x_i \in \mathbb{R}^{D}$, $i = 1, 2, ..., N$, independently chosen uniformly from $[-1, 1]$. And we have $y_i = \prod_{k \in [K]} \mathrm{sgn}(w_k \cdot x_i)$, where $\mathrm{sgn}$ is the sign function.

While learning a single halfspace $K=1$ is easily solved by a two-layer network with ReLU activation, it becomes much more difficult for neural networks to learn when $K$ grows. This can be observed in Figure-\ref{fig:halfspaces}, where a 3-layer neural network is used to learn the intersections.

For a modular approach, we follow Algorithm-\ref{alg:halfspace}, which is a simplified version of Algorithm--\ref{algo:informal_execution-loop} and it probabilistically route to sub-modules. The input data are batches of triplets $\{(k, x_i^k, y_i^k)\}_{i \in [B]}$, where $B$ is the batch size, $k \in [K + 1]$ is the task id, $x^1_i = ... = x^{K+1}_i = x_i$,  $y_i^k = \mathrm{sgn}(w_k \cdot x_i)$ for $k \in [K]$ and $y_i^{K+1} = \prod_{k \in [K]} y_i^k$, and we maintain a task list $T$ and module list $\Phi$.


The results are plotted in Figure-\ref{fig:halfspaces}, with $K=1, 2, ..., 10$,  $D = 100$,  $p_{\mathrm{atomic}} = 0.5$ and $ p_{\mathrm{compound}} = 0.75$. For the modular approach, all the modules are $3$ layer fully-connected network of the same size and are trained for $10$ epochs. For the end-to-end approach, a single $3$ layer fully-connected network with $10$x hidden units of the modular models and is trained until convergence.  We observe for large $K$ ($K \geq 7$ in the figure), the end-to-end approach fails at the task while the modular approach continues to have good performance. See appendix for more details.


%
\begin{figure}
    \vspace{-3mm}
    \centering
    
    \begin{minipage}[b]{.45\textwidth}
    \begin{algorithm}[H] 
    \caption{Probabilistic routing algorithm}\label{alg:halfspace}
 	\begin{algorithmic}[1]
 	    \Statex \textbf{Input}: Batches of $\{(k, x_i^k, y_i^k)\}_{i \in [B]}$, $k$ the same within a batch.
	    \Statex \textbf{Constants}: $p_{\mathrm{atomic}}$, $p_{\mathrm{compound}}$.
	    \Statex \textbf{Initialization}:  Set of modules $\Phi = \emptyset$, set of task ids $T = [K+1]$
	    \Statex Repeat the following steps:
 	     \Statex \hspace{2pt} 1. w.p. $p_{\mathrm{atomic}}$,  train an atomic module $\phi_k$ that maps $x_i^k$ to $y_i^k$ (note we keep a separate copy of $\phi_k$ for each different DAG structure based on iteration choices in step 2 from the original input to $x_i^k$). If training succeeds, set $\hat{\phi}_k$ to be the DAG upto $\phi_k$ and add it to $\Phi$.
 	    \Statex  \hspace{2pt} 2. w.p. $1 - p_{\mathrm{atomic}}$, for each $\hat{\phi} \in \Phi$, w.p. $p_{\mathrm{compound}}$, set $x_i^k \leftarrow \mathrm{concat}(x_i^k, \phi(x_i^k))$.

 	    \Statex \textbf{Return} $M$.
 	\end{algorithmic}
    \end{algorithm}
    \end{minipage}\hfill
    \includegraphics[width=.54\textwidth]{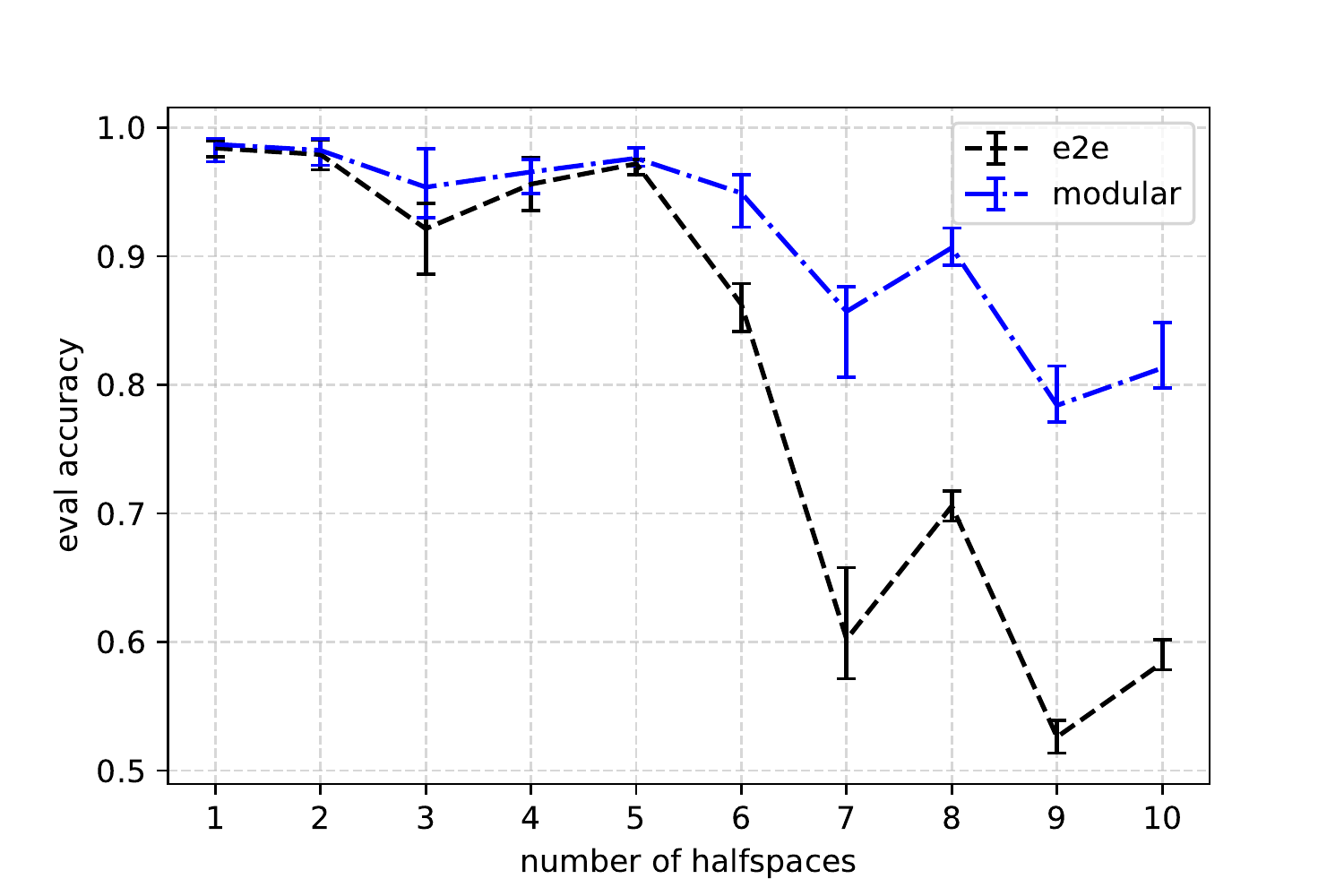}
    \caption{Intersection of halfspaces \textbf{Left:} Pseudocodes for the modular approach (``w.p." is abbreviation for ``with probability").
    \textbf{Right:} Modular approach continuous to have good performance while end-to-end approach fails to learn for $K \geq 7$.}
    \label{fig:halfspaces}
    \vspace{-3mm}
\end{figure}

\subsection{Five digit recognition}
In this experiment, we compare the ``end to end" approach and a modular approach for the task of recognizing $5$-digit numbers, where the input is an image that contains $5$ digits from left to right, and the expected output is the number that is formed from concatenating the $5$ digits. This task is described in Example 2 of Appendix F. Note in this task, we have $3$ sub-tasks: task-1 is single digit recognition, task-2 is image segmentation, and task-3 is $5$ digit recognition. 

For the ``end to end" approach, we train a convolutional neural network to predict the 5 digit number (see appendix for more details). For the modular approach, the input data are batches of triplets $\{(k, x_i^{(k)}, y_i^{(k)})\}_{i \in [B]}$, where $B$ is the batch size, $k \in [3]$, $x_i^{(1)} = x_i^{(2)} = x_i^{(3)} =x_i$ is the image. $y_i^{(1)}$ is the single digit label, $y_i^{(2)}$ is $5$ segmentation coordinate pairs (upper-left and lower-right coordinates), and $y_i^{(3)}$ is the 5 digit number label. We also maintain a task list $T$ and module list $\Phi$.  For training an atomic module in Algorithm-\ref{alg:halfspace}, we only allow the module to take inputs of the same modality (i.e. either only image or only digits, discarding the others).

We construct the training and test datasets by concatenating $5$ images from the MNIST dataset. The results of the two approaches are compared in Table-\ref{tab:digits}. We observe the modular approach achieves higher accuracy and has less variance with the same training steps.

\begin{table}
    \centering
    \vspace{2mm}
    \begin{tabular}{@{}ccccc@{}}
    \hline
    method & accuracy & steps (1-digit) &  steps (segmentation) & steps (5-digits) \\
    \hline
    end-to-end   & $74.5 \pm 4.5 $  \%  & NA & NA & $18760$  \\
    modular   & $92.0 \pm 0.5$\% & 2560 & 640 & $18760 \pm 9380$   \\
    \hline
    \end{tabular}
    \vspace{-1mm}
    \caption{Comparison of end-to-end and modular algorithms for 5-digits recognition: accuracy and number of training steps for different tasks to succeed. Note here each step is processing one batch with a batch size of $128$, and we consider a task successful if the accuracy is above $90\%$.}
    \label{tab:digits}
    \vspace{-5mm}
\end{table}

\section{Discussion and Future Work}
We saw a uniform continual learning architecture that learns tasks hierarchically based on sketches, LSH, and modules. We also show how a knowledge graph is formed across hash buckets as nodes and formally show its utility (for e.g. for finding common friends in a social network) in Appendix~\ref{app:knowledgegraph}. Extensions of the decision tree learning to solve reinforcement learning tasks are shown in Appendix~\ref{sec:architecture-v3}. Although our inputs were labeled with a unique task description vector, we note our architecture works even with noisy but well-separated contexts (see Appendix~\ref{sec:appendixarchitecturev0}). A weakness of our work is that we have ignored how logic or language could be handled in this architecture -- while perhaps there could be separate compound modules for those kinds of tasks, we leave that topic open.

\section*{Acknowledgements}
ZD and BJ were supported by NSF grants CCF-1718380 and IIS-1908287.


\newpage

\appendix
\section{Locality Senstive Hashing}\label{appendix:lsh}
{\bf Locality Sensitive Hashing} (LSH) is a popular variant of hashing that tends to hash similar objects to the same buckets. Let us look at an LSH that maps an input to one (or a few locations) out of the  $m$ hash buckets. It is well-known that LSH provably provides sub-linear query time and sub-quadratic space complexity for approximate nearest neighbor search. More specifically, fix $0 < r_1 < r_2$, where $r_1$ is the threshold for nearby points, and $r_2$ is the threshold for far-away points, i.e. for $x, y \in \mathbb{R}^d$, we say $x$ and $y$ are nearby if $|x - y|_2 \leq r_1$ and they are far-away if $|x - y|_2 \geq r_2$, where $|x|_2$ is the $2$-norm of the vector $x$. Let $c = r_2/r_1 > 0$ denote the distance gap as a ratio. Let $p_1 \le Pr(h(x) = h(y): |x - y|_2 \leq r_1)$ and $p_2 \ge Pr(h(x) = h(y): |x - y|_2 \geq r_2)$ denote lower and upper bounds on the collision probability of nearby points and far-away points, respectively. Define $\rho = \frac{\log (1/p_1)}{\log (1/p_2)}$. Then LSH-based nearest neighbor search has a $O(n^{\rho})$  query time and $O(n^{1 + \rho})$ space complexity for a $c$ approximate nearest neighbor query \cite{andoni2015practical, andoni2014beyond, andoni2015optimal}.

One example LSH function uses random hyperplane based LSH \cite{charikar2002similarity} to map a vector into a hash bucket, although other types of hashing such could be used as well -- for example min-hash \cite{broder1997resemblance} could be used on a set or a tuple object to map that object to a discrete hash bucket. 

\section{Architecture}\label{appendix:architecture}

\subsection{Sketches Review}\label{appendix:sketching}

Our architecture relies heavily on the properties of the sketches introduced in \cite{ghazi2019recursive}. In this section we briefly describe some of the key properties of these sketches; the interested reader is referred to \cite{ghazi2019recursive,panigrahy2019does,sketchmem2021} for the full details.

A sketch represents any event, an input or an output at a module. It may represent an ``object'' that may recursively contain a (unordered)set or a (ordered)tuple of sketches. 

Any input or output of a module can be represented by a {\em sketch}. 
For example an input image has a sketch chat can be thought of as a tuple [IMAGE, \la bit-map-sketch\ra ]. An output by an image recognition module that finds a person in the image can be represented as [PERSON, [\la person-sketch\ra , \la position-in-image-sketch\ra ]); here IMAGE, PERSON can be thought of as a ``labels''. However the sketch may be more complicated like an object for example the \la person-sketch \ra could in turn be set of such pairs 
\{[NAME,\la name-sketch \ra], [FACIAL-FEATURES,\la facial-sketch \ra], [POSTURE,\la posture-sketch \ra]\}. Thus a sketch could be represented as a tree. Further there may be compound sketches that consist of a set of sketches. For example an image consisting of multiple people could be a set \{\la person-1-sketch \ra, \la person-2-sketch \ra,..,\la person-$k$-sketch \ra\}.


{\bf Sketches can be used to backtrack the chain of modules that produced it}: An output sketch may also recursively point to the input sketch and the modules it came from, e.g. 
recursive-sketch(output) = \{[OUTPUT-SKETCH,\la output-sketch \ra], [MODULE-ID,\la module-id \ra], [RECURSIVE-INPUT-SKETCH, \la recursive-input-sketch \ra]\}. By keeping recursive-input-sketch to some depth, we can find find the entire tree or DAG of modules that produced this output sketch. A method for representing such structured sketches as a dense vector using subspace embeddings (each object sketch is embedded into a random subspace for that type of object) is provided in \cite{ghazi2019recursive}. There is way to sketch the outputs of a modular network so that similar finding lead to similar sketches; the main idea is that similar input phenomena will cause almost the same set of modules to fire with similar output embeddings. See \cite[Theorem2]{ghazi2019recursive}.

{\bf Types are encoded in subspaces}: An object of a particular ``type'' is represented by a sketch that embeds it in a specific random subspace that uniquely determines the type.
A set of ``type, value" pairs can be sketched by packing each type in a separate subspace by using random matrices (the actual distribution is more complicated to prove stronger robustness guarantees see \cite[Theorem 1]{ghazi2019recursive}).

{\bf Dense representations of sketches}: As described above, an object containing sub-objects of types $T_1, T_2, T_3$ can be represented by the set 
$s = \{[T_1,s_1], [T_2,s_2], [T_3,s_3]\}$ where $s_1, s_2, s_3$ are sketches of the sub-objects. In \cite{ghazi2019recursive} a method is given for converting this into a dense representation $r(s)$, which we summarize here.

A dense representation $r(s)$ of this can be obtained recursively as
$r(s) = R_{T1} r(s_1) + R_{T2} r(s_2) + R_{T3} r(s_3) $ where the $R's$ are random matrices that depend on the type $T_i$ with output dimension large enough to recover the sub-sketches. The $R$ is drawn from a distribution given by $(I + R')/2$ where $R'$ has mean $0$ (the exact distribution can be found in \cite{ghazi2019recursive}). This ensures that $r(s)$ has some similarity to $r(s_1), r(s_2), r(s_3)$.
Thus a compound sketch has some ``similarity'' to each of its components. A sketch is recursive in the sense that it is a compound sketch of all its components/subtrees -- lower level subtrees get exponentially decreasing weight (see \cite[Theorem 1]{ghazi2019recursive}). Any component sketch with high enough weight can be recovered. Further those with weights below a threshold may be retrieved from buckets in the hash table (see section~\ref{sec:storinglargeobjects}). Also, from the compound sketch of a large number of sketches the average value of the component-sketches can be recovered (see Claim 5 in \cite{ghazi2019recursive}). 
A tuple $[s_1, s_2, s_3]$ can simply be thought of as the set $\{ [1,s_1], [2,s_2],[3,s_3] \}$.

A set of sketches of the same type can be sketched by using a local LSH table. The set of sketches landing at each bucket is sketched recursively. This gives an array of sketches. The sketch of this array is the final sketch.

Note if the if the set very large, we will not be able to recover the sketch of each of its members but only get a ``average" or summary of all the sketches -- however if a member has high enough relative weight (see \cite[Section 3.3]{ghazi2019recursive}) it can be recovered.

\subsubsection{Storing large objects}\label{sec:storinglargeobjects}

Large objects such as long strings can be stored as compound-sketch that is sketched recursively into smaller and smaller sequence of sketches. Memory of a sequence of events can be stored as sketches in buckets that link to each other that can be retrieved later when it needs to be replayed. A string of length $n$ can be sent to a CNN that uses patches of size $s$ with stride of $s/2$, producing $2n/s$ patches and their corresponding sketches. These sketches may be stored in a hash table. These $2n/s$ patch sketches could further be sketched in the same way till we get a single compound sketch at the top. This ``tree'' of sketches can be implicitly stored in a hash table. The final top sketch serves as a summary of the entire string -- it can be used to find substrings that have very high frequency -- for example if a patch occurs a large fraction of times that can be inferred from the top level sketch even without looking at the rest of the sketches in the tree. The sketch of a large object can implicitly be used as a {\em pointer} to that object.

Programs can also be viewed as strings of instructions or strings of matrices. By using the above method large programs can be stored and accessed in the hash memory.

\subsection{Architecture Principles}\label{sec:appendixarchitectureprinciples}

The following generalizes the architecture principles and algorithm~\ref{algo:informal_execution-loop} to include knowledge graph edges that keep track of frequent associations (see Appendix \ref{app:knowledgegraph} for applications of such associations and Appendix \ref{sec:appendixarchitecturev2} for RL applications)

\begin{enumerate}
\item Sketches.
\begin{itemize}
\item All phenomena (inputs, outputs, commonly co-occurring events, etc) are represented as sketches.
\item There is a {\em function from sketch to context} $f: S\rightarrow C$ that gives a coarse grained version of the sketch. This is obtained by looking at the fields in the sketch $S$ that are high level labels and dropping fine details with high variance such as attribute values; it essentially extracts the ``high-level bits'' in the sketch $S$.
\end{itemize}
\item Hashtable indexed by context that is robust to noise.
\begin{itemize}
\item The hash function $h: C\rightarrow \text{[hash-bucket]}$ is ``locality sensitive'' in the sense that similar contexts are hashed to the same bucket with high probability.
\item Each hash bucket may contain a trainable program $P$, and summary statistics as described in Figure~\ref{fig:hashbucket}. We don't start to train $P$ until the hash bucket has been visited a sufficient number of times.
(Note: A program may not have to be an explicit interpretable program but could just be an ``embedding'' that represents (or modifies) a neural network.)\end{itemize}
\item Routing-module (OS).
\begin{itemize}
\item Given a set of sketches from the previous iteration, the routing module identifies the top ones, applies the $f$ function followed by $h$ to route them to their corresponding buckets. Before applying $f$ it may use {\em attention} to combine certain subsets of sketches into a compound sketch. 

\end{itemize}
\item Knowledge graph of edges.
\begin{itemize}
\item Information about frequently co-occurring sketches (e.g. if sketch $S_1$ is frequently followed by sketch $S_2$) is stored as edges connecting hash table buckets that form a knowledge graph.


\item When the routing module visits a bucket, in addition to the program $P$, it can also extract the sketches on the outgoing edges at that hash bucket. One could also view the program $P$ as the ``default edge'' at that bucket.
\end{itemize}
\end{enumerate}

\begin{algorithm}[t]
\SetStartEndCondition{ }{}{}%
\SetKwProg{Fn}{def}{\string:}{}
\SetKwFunction{Range}{range}
\SetKw{KwTo}{in}\SetKwFor{For}{for}{\string:}{}%
\SetKwIF{If}{ElseIf}{Else}{if}{:}{elif}{else:}{}%
\SetKwFor{While}{while}{:}{fintq}%
\DontPrintSemicolon%
\AlgoDontDisplayBlockMarkers\SetAlgoNoEnd\LinesNumbered%
\SetAlgoSkip{}

\KwIn{input sketch $T$ (this sketch may contain a desired output for training)} 
{\tt current-sketches} $\leftarrow \{T\}$ \;
\While{{\tt current-sketches} is not empty}{
{\tt current-programs} $\leftarrow$ $\emptyset$\;
\ForEach{sketch $S$ in {\tt current-sketches}\, }{
extract context $C = f(S)$ \;

update access-frequency-count of bucket $h(C)$ \;

Store $S$ as an outgoing edge of $h(C)$, if there are too many sketches store a compound sketch. Store pointers to co-referencing/co-occurring sketches buckets. \;

\uIf{bucket $h(C)$ has a program $P$}{append $(S,P)$ to {\tt current-programs}
}
\uElse{
\uIf{bucket $h(C)$ is frequently accessed}
    {initialize program at $h(C)$ with program from nearest non-empty context bucket and mark it for training}
    
fetch programs from nearby trained buckets (with similar contexts), append those $(S,P_i)$ to {\tt current-programs}}}

Routing module chooses some subset of {\tt current-programs}, runs each program on its associated sketch, appends outputs to {\tt current-sketches}\;

Append sketches on outgoing edges of accessed buckets to {\tt current-sketches}\;

\uIf{any of the programs are marked for training}
    {routing module picks one or some of them and trains them, and may choose to stop execution loop}

\uIf{any of the sketches is of (a special) type OUTPUT or ACTION sketch}
    {routing module picks one such, outputs that sketch or performs that action, and may choose to stop execution loop}

\uIf{any of the sketches is of type REWARD sketch (say for correct prediction or action)}
    {routing module updates the reward for this bucket and propagates those rewards to prior buckets}    

Routing module uses attention to combine elements of {\tt current-sketches} into at most $k$ compound sketches $S_1, \dots, S_k$ (may produce 0 sketches)\;

{\tt current-sketches} $\leftarrow \{S_1, \dots, S_k\}$\;
}
\caption{Informal presentation of the main execution loop 
}
\label{algo:execution-loop-full}
\end{algorithm}

The system works in a continuous loop where sketches are coming in from the environment and also from previous iterations; the main structure of the loop (recall Figure~\ref{fig:architecture-and-sketch}) is:

\adjustbox{scale=0.8,center}{
\begin{tikzcd}[every matrix/.append style={nodes={font=\small}}]
Phenomena \rar{input} & sketch \rar{f} & context \rar{h}  & bucket \rar  & program  \arrow[bend right=30, swap]{lll}{produces} \rar{{\tiny output}} & Phenomena
\end{tikzcd}
}

Our architecture can be viewed as a variant of the Transformer architecture \cite{radford2021learning, shazeer2017outrageously}, particularly the Switch Transformer \cite{fedus2021switch} in conjunction with the idea of Neural Memory \cite{sketchmem2021}. Instead of having a single feedforward layer, the Switch Transformer has an array of feedforward layers that an input can be routed to at each layer. Neural Memory on the other hand is a large table of values, and one or a few locations of the memory can be accessed at each layer of a deep network. In a sense the Switch Transfomer can be viewed as having a memory of possible feedforward layers (although they use very few) to read from. It is viewing the memory as holding ``parts of a deep network'' as opposed to data, although this difference between program and data is artificial: for example, embedding table entries can be viewed as ``data'' but are also used to alter the computation of the rest of the network, and in this sense act as a ``program modifier''.

New modules (or concepts) are formed simply by instantiating a new hash bucket whenever a new frequently-occurring context arises, i.e. whenever several sketches hash to the same place; the context can be viewed as a function-pointer and the sketch can be viewed as arguments for a call to that function. Frequent subsets of sketches may be combined based on {\em attention} to produce {\em compound} sketches. Finally we include pointers among sketches based on co-occurrence and co-reference in the sketches themselves. These pointers form a knowledge graph: for example if the inputs are images of pairs of people where the pairs are drawn from some latent social network, then assuming sufficient sampling of the network, this network will arise as a subgraph of the graph given by these pointers. The main loop allows these pointers to be dereferenced by passing them through the memory table, so they indeed serve the intended purpose.

Thus external inputs and internal events arrive as sketches that are converted into a coarser representation using the $f$ function that gets mapped to a bucket using hash function $h$; the program at that bucket is executed to produce an output-sketch that is fed back into the system and may also produce external outputs. 
This basic loop is executed by the routing-module which can be thought of as the operating-system of the architecture. In each iteration the routing-module gathers the sketches output from the modules executed in the previous rounds, along with the input sketches from the environment and retains the top $k$ based on some notion of weight/importance (this could be a combination of frequency and rewards, which is tracked in the buckets corresponding to the sketches). It may also use attention to combine certain subsets of these. These are then routed using the $f$ followed by the $h$ function to their respective modules in the hash buckets. The programs in these buckets execute the corresponding sketches producing new sketches (these new sketches may also produce outputs or actions into the environment) that are sent back into the current collection of sketches. Each bucket also tracks other co-occurring/co-referenced sketches which may also be retrieved when that bucket is visited.
In Algorithm \ref{algo:execution-loop-full} we have under-specified and left out how the routing module makes the discrete choices. We will show a simple method is to implement it as a decision tree that makes probabilistic choices that eventually converge to an optimal set of deterministic choices (see DecisionTree Algorithm.\ref{algo:fdecisiontree}). 

{\bf Hash function $h$}:
The hash function $h$ is an LSH function, so similar contexts are hashed to the same bucket. When the model encounters a sketch whose context is unfamiliar (i.e. is sufficiently far away from any existing contexts) a new hash bucket is instantiated for that context. Each bucket contains (see Figure ~\ref{fig:hashbucket}):

\begin{figure}
\centering
\includegraphics[trim = 150 80 260 70 , clip, height =3cm]{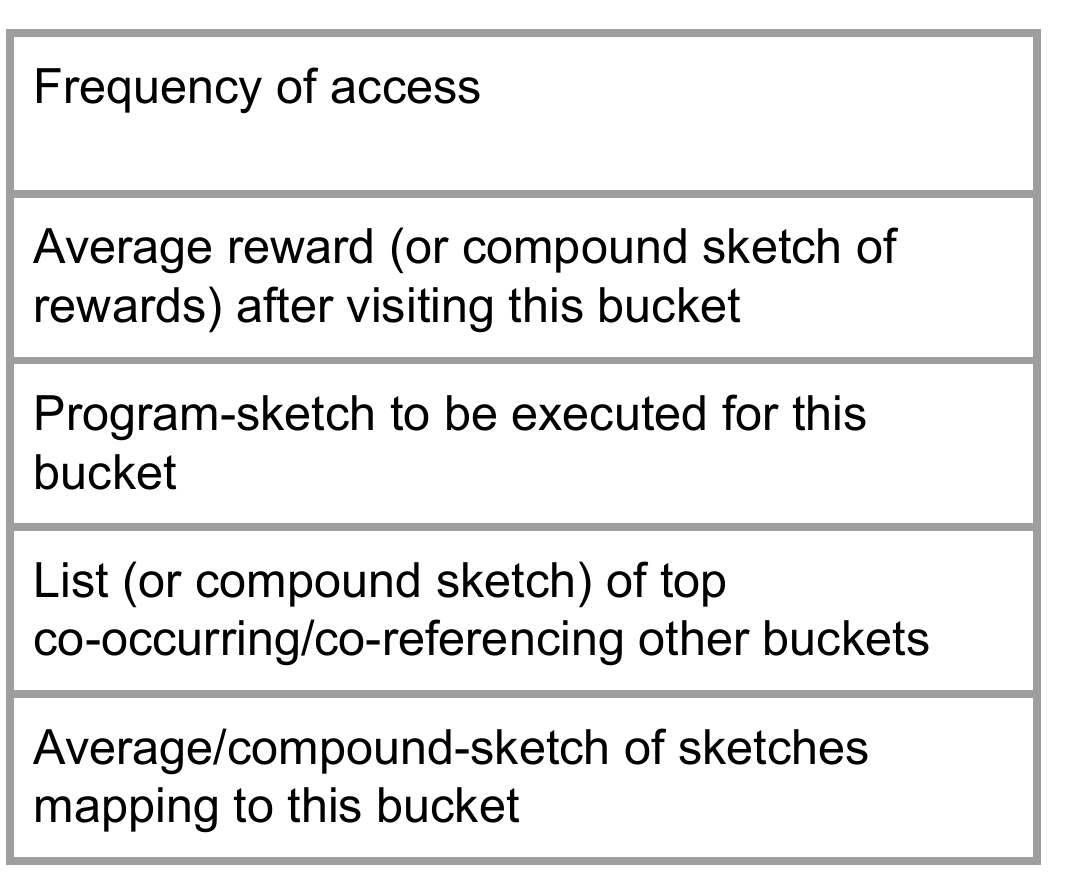}
\caption{Fields in a hash bucket, including the program-sketch (if it exists), pointers to related buckets, and summary statistics such as frequency of access.
The average or compound sketch of all the sketches mapping to this bucket can be used to identify the different frequent pathways that led to this bucket. The average or compound sketch of all the rewards-sketches can be used to identify the different frequent pathways from this bucket that lead to rewards and their reward values. 
}
\label{fig:hashbucket}
\end{figure}

\begin{enumerate}
\item The program $P$ which is a learned a function, e.g. a trainable neural net, for that bucket.
\item summary statistics, e.g. frequency counter, reward/quality score.
\item summary of sketches that point to this bucket. The summary could be a compound sketch of things mapping here. Note that the compound sketch contains information about the average value (see Claim 5 in \cite{ghazi2019recursive}).
\item local information about the knowledge graph, e.g. outgoing edges from that bucket.
\end{enumerate}

{\bf Handling hash bucket collisions from $h$}:
To handle collisions, if instead of using one hash function $h$ if we use $r$ of them (for some integer $r \geq 1$) then with high probability for each sketch there will be at least one distinct bucket (in fact at least a constant fraction will be distinct) as long as its context is far enough from that of all others. In case many similar contexts hash to the same bucket, the that bucket will have a high frequency count. In that case when the routing module encounters such a bucket, it could use additional LSH bits to rehash such sketches to new buckets which is likely to put them into different buckets -- this could be repeated until we get to buckets that have bounded frequency counts.

{\bf Task}: A task refers to logically coherent subset of training examples from the external world with a specific processing to be applied to each of those examples to produce a desired output for each of them. For example in the simplest case each example may come with a specific task-description sketch or identifier that specifies the task. However the task description may not be explicit in the input, but may be identified after routing the input through a sub-dag of modules. Each task maps to a unique context which is determined by applying the $f$ function on the sketch at some level.

{\bf Learning the program in a bucket}: Each bucket contains a trainable function from a certain class (such as neural networks of a fixed small depth). More generally it could represent a vectorized ``embedding'' that modifies another global network or produces another network. This function may depend on the output of other buckets as prerequisite inputs. 

{\bf Module}: A module refers to a program in a bucket. This may either be an atomic Module  or a compound Module that can call other atomic or compound modules (thus it is a sub-dag over other modules). Thus a compound module in a bucket may  {\em recursively} ``call'' functions in other buckets realized as a frequently seen computation DAG of sketches flowing through modules. For example the entire architecture can be thought of as one giant compound module defined by the initial module where all input sketches are sent (think that there is some default ``boot" module where all inputs are sent; this boot module iteratively calls the routing module and modules in the hash buckets where derived sketches get routed -- note that iterative calls can also be implemented as a tail recursion where the first iteration recursively hands off the processing for the remaining iterations.
 
As the routing module explores, it records sequences of modules that led to high reward at this bucket by sketching the path and storing it as an outgoing edge in the knowledge graph. Over time, this edge ``hardens'' into the default path for this task: it becomes so high weight compared to the other edges that $f$ automatically focuses on it and follows it deterministically. We call this the ``program'' for this bucket. 

{\bf Attention}: We use ``attention'' in a very broad sense, meaning not just the mechanism as it appears in e.g. transformer architectures but more generally as a method of combining sketches based on pairwise similarity and/or relevance into a weighted tuple/set. We could use attention to extract components from one sketch based on another sketch and/or edges between their buckets. 
This attention can be used for example to connect the spoken name of a person to their image in a group photo via an edge in the knowledge graph that captures the frequent co-occurrence of the spoken name with the face. For example, imagine a picture of many people with one of their names being spoken. First the picture goes to visual module, which identifies that there are faces and sends it to facial-recognition module; this finds multiple familiar faces, and the bucket of one of these faces has an existing outgoing edge pairing it with the audio of the name. The similarity between these sketches is reflected in the weights generated by the attention module and the result is a combined sketch connecting this face in the picture with the audio input.

{\bf Routing module}: The routing module applies function $f$ that maps sketches to context followed by hash function $h$ that maps the context to a bucket. The function $f$ can be viewed as extracting a coarse representation of the sketch by extracting stable fields such as labels and dropping high variance ones. Since there may be several options in designing $f$ for a certain type of sketch, there may be some exploration where it makes probabilistic choices and later converges to a specific choice. For example there may be some probabilistic choices at each step regarding which components of a sketch to keep before routing the sketch to a hash bucket. For now we assume that routing module starts with a very simple set of rules and refines its probability distribution for each bucket over time based on success or failure of its choices (e.g. whether the loss score for that example is below some threshold).

We start with a very simple definition of the function $f$  in version v0 (section~\ref{sec:v0}) that simply drops certain fields in the sketch. In section~\ref{sec:v2.decision-tree}  we show how this can be formalized in the framework of our architecture by viewing this operation that picks a subset of the fields in an input sketch as yet another modular decision tree task by using sub-modules.

{\bf Knowledge graph}: The knowledge graph is implicitly forms by the pointers from a bucket to other frequently co-occurring buckets -- if there are too many such we may retain only the top few in addition to storing a compound sketch of the co-occurring sketches. Although in the architecture principles we said that knowledge graph edges are formed by buckets of frequently co-occurring sketches $S_1, S_2$ pointing to each other, we can also achieve this by simply creating a compound sketch for the frequently co-occurring pair $[S_1, S_2]$ and having edges between co-referencing sketches $[S_1, S_2]$ and $S_1, S_2$ each; thus all edges would be co-referencing edges. [[Rina: make this precise. into a definition]]. Thus by simply creating compound sketches and having pointers between the compound sketch and the component sketches and vice versa we automatically capture frequent co-occurence -- note that any sketch (including compound sketches) are persisted in hash buckets only when they occur with a frequency that exceeds a certain threshold (see details in a few paragraphs below)

{\bf Mature and immature buckets}: 
For simplicity we may think of some buckets  whose quality score is above a certain (user-defined) threshold as mature bucket that are marked as ``well trained". The parameters of the bucket's function may still be updated and further refined as new sketches arrive, but its main program is frozen and no further exploration is required in terms of solidifying choices of the routing module (say for the $f$-function) in handling sketches that map to this bucket. An immature bucket is one that is not fully trained: the routing module may not have found a good sequence of modules for this task yet, may not have figured out good choices for routing sketches and applying the $f$ function for sketches that map to this bucket, and/or the program in the bucket may not have been fully trained. Other modules cannot call an immature module as part of their program.

{\bf Backpropagation}: The parameters of the neural net in a bucket are updated whenever the example is evaluated in that bucket, that is whenever the routing module decides to stop exploring and train in that bucket. If the loss in the bucket is below some threshold then the knowledge graph is also updated, with a copy of the sketch being recorded as a co-occurring sketch for all modules in the execution pathway. 

{\bf Distinguished modules}: We assume that the architecture is provided with certain basic ``hardcoded'' modules where necessary, for example specialized audio and image processing modules with pre-trained CNNs to extract raw audio-visual data and embed it into a representation space. This is by analogy to humans, who have to learn to interpret data from their senses but don't have to evolve from scratch the concept of eyes.  We may also assume the existence of an ``input module'' that is the first module where all external inputs such as images, audio, text are first routed to. This module may separate modalities from the input sketch which may individually get routed to specific modules to process images, audio, text separately. The modules for processing images, audio may in-turn find text in the images, sounds and that text may get routed to text module.

{\bf Only frequent contexts are persisted}: If there are $m$ hash buckets in the LSH table we will only track contexts that appear at least with frequency $O(1/m)$, while others get ``timed out" and eventually forgotten as they will not appear with sufficient frequency --  we assume $m$ is at least the number of distinct contexts that need to be trained to correctly learn all the tasks. Note that this tracking of frequency of persistent and ephemeral contexts to ensure we catch anything with frequency at least $O(1/m)$ can be done in a total of $\tilde O(m)$ buckets -- one way to achieve this is to simply drop (time-out) from the system any context that does not appear within a time interval of $\tilde O(m)$; clearly, this way only $\tilde O(m)$ contexts are ever in the system at any one time.

{\bf A tail recursive view of the execution loop}: We can think of a tail recursive variant of Algorithm \ref{algo:execution-loop-full} where the loop is replaced a recursive call to itself at the end. In this case, the inside of the loop ``foreach sketch $S$" is replaced by a recursive call to Algorithm with input $S$ . This recursive view is also useful for analysis in certain cases. In some cases the algorithm may execute only the content inside the foreach loop involving applying the $f$ and the $h$ functions on the input sketch which could correspond to executing the leaf level of the recursion. The depth of the recursion (or the number of iterations) may be capped at some upper limit to prevent infinite loops. (Also see section \ref{app:misc} for a recursive view of a compound module)

\begin{algorithm}[t]
\SetStartEndCondition{ }{}{}%
\SetKwProg{Fn}{def}{\string:}{}
\SetKwFunction{Range}{range}
\SetKw{KwTo}{in}\SetKwFor{For}{for}{\string:}{}%
\SetKwIF{If}{ElseIf}{Else}{if}{:}{elif}{else:}{}%
\SetKwFor{While}{while}{:}{fintq}%
\DontPrintSemicolon%
\AlgoDontDisplayBlockMarkers\SetAlgoNoEnd\LinesNumbered%
\SetAlgoSkip{}

\KwIn{input sketch $T$ (this sketch may contain a desired output for training)} 

With some probability set leaf-level-recursion = TRUE \;

\uIf{leaf-level-recursion} 
{ 

    extract context $C = f(T)$ \;
    
    update access-frequency-count of bucket $h(C)$ \;
    
    \uIf{bucket $h(C)$ has a program $P$}{
        if P is marked for training, train it \;
        
        return P(S)
    }
    \uElse{
        \uIf{bucket $h(C)$ is frequently accessed}{
            initialize program at $h(C)$ with some random program
            and mark it for training.\;
            
            with some probability {
                Fetch programs $P_i$ (possibly by some similarity criterion) \;
                return list of $P_i(S)$ \;
            }
        }
    }
}
\uElse{ 
    \ForEach{component sketch $S_i$ in T \, }{
    
    T = list of Algorithm1($S_i$)
    }
    T = pick  $k$ combinations of sketches in  T, and combine them into compound sketches:  
    
    return Algorithm(T)
    
}
\caption{Informal presentation of the recursive view of the main execution loop 
}
\label{algo:informal_execution-loop-recursive}
\end{algorithm}


\section{Architecture v0}\label{sec:appendixarchitecturev0}

\begin{claim}\label{claim:independent-tasks}
Given an error rate $\epsilon>0$ and confidence parameter $\delta>0$ and $n$ independent tasks, each of which require at most $M=M(\epsilon,\delta)$ examples to learn to accuracy $1-\epsilon$ with probability $1-\delta$, and training data as described in above, with probability $1-(n+1)\delta$, Architecture v0 learns to perform all $n$ tasks to accuracy at least $1-\epsilon$ in $O(Mn\log \frac{n}{\delta})$ steps.
\end{claim}
\begin{proof}
This follows from the fact that the problem essentially breaks down into $n$ separate supervised learning tasks. In the learning algorithm we simply route each sample using $f$ and $h$ to its corresponding task according to its task descriptor and use the learning algorithm $\mathcal{A}_{\mathcal{M}}$ to train the function $\hat{m}_t$ in the corresponding hash bucket.  
The algorithm for v0 falls into the framework of Algo.1. However, $f$ and $h$ are restricted and other routing module operations become NO-OP. Because in v0 the tasks are independent. 
\end{proof}

Here, we assumed that each task has a fixed, unique task descriptor. Using the locality-sensitive hash function, it is straightforward to extend v0 slightly to the case where each task is represented by noisy but well-separated task descriptors.

{\bf Noisy contexts}: Although we have been thinking of contexts as precise and fixed for a task, We can also relax the assumption that the context of a task should be identical each time, instead allowing some noise in the contexts. Our architecture can handle noisy contexts as we use an LSH table; we can easily replace the LSH-function with an $r$-LSH that makes $r$ different hash function for some $r$. Now each context accesses $r$ buckets, and programs can be encoded in a distributed/replicated fashion to render the contexts robust to noise. The following two Theorems show how the relevant information for a context can be stored in a distributed robust manner (like in an error correcting code) so that even having access to a fraction of the locations where it is encoded is sufficient to correctly recover the information -- this allows us to index the information using a "corrupted" 
version of the context.

\textbf{Assumption}:
Suppose there is desired sketch $v^*$ and the noise procedure that produces $v$ such that: 
$$
v \sim \mathcal{N}\left(v;  v^*, diag(\beta)\right).
$$

\begin{theorem}
If there is a ball of points of radius $\delta$ in sketch space so that all those points should go to the same program then for $r \ge n^{O(\delta)} $ the program will get programmed in {\em any} of the several hash locations the points map to (as long as the points are picked randomly) from the ball.
\end{theorem}
\begin{proof}
Since we are using a locality sensitive hash table if we use $r$-LSH functions, the context doesn't need to point to the same set of buckets -- but as long as there is at least one common bucket it can retrieve the program. LSH guarantees that as long as the contexts $C$ and $C'$ are within $\delta$ distance, they will go to the same bucket with probability at least $n^{-O(\delta)}$. So if $r  \gg n^{O(\delta)}$ with high probability there will be an intersection. During inference we can look at all the non-empty buckets and take the average of the programs stored in all those buckets. During training if we add a regularizer that minimizes the sum of the norms of the program vector representation, then all the programs in the $r$-buckets will go to the same value.
Thus if a task is trained using large number of hashed buckets then it is highly resilient to change in context as all that is needed is for a few of the hash locations to intersect. This proves the Theorem.
\end{proof}

Thus even though the different contexts go to different sets of buckets those buckets contain the same program; this program sketch now becomes an identifier/common-sketch for this unique common context across these noise contexts.

{\bf Distributed storage of programs}: In fact, a program need not fit entirely in one bucket but may be assembled in a robust manner from the $r$-buckets. Thus a program may be stored across multiple hash buckets so that any small subset of them could be used to recover the program. Let us say the program is $s_p$ and the amount of program-field in each bucket is $s_b$. We will show how from a random subset of $l$ out of the $r$-buckets for this context is sufficient to assemble the program as long as $l > \tilde \Omega(s_p/s_b)$. The main idea is to associate each bucket $i$ with  a random sparse rotation matrix $R_i$. Then if a large set  of $r$ locations are trained to store a particular program $y$, any small subset  $\{i_1,.., i_l\}$ of those locations may be sufficient to read $y$. That is, $y = (R_{i_1} x_{i_1} +..+ R_{i_l} x_{i_l})/l$ where $x_i$ is the value stored at bucket $i$. This idea may also be used to store a program in a distributed fashion across entirely different contexts.

\begin{theorem}
There is a way to store a program $y$ in a distributed manner across $r$ buckets so that any random subset of $l$ of these buckets can be used to reconstruct the $y$, as long as $l > \tilde \Omega(s_p/s_b)$. 
\end{theorem}
\begin{proof}
First look at the case where $s_p = s_b$.
For simplicity think of each rotation matrix as identity. Then we will show that at local minima all program-pieces $R_i x_i$ are identical. This is achieved by adding a regularizer that minimizes the sum of the norms of the program pieces $R_i x_i$  in the different buckets. The same argument holds if the matrices are full rank.

If $s_p > s_b$ first lets look at the limiting case when $s_b = 1$. So we are taking a set $l$ numbers and using it to get an $s_p$-dimensional vector $y$. This can be done by using a random sparse $s_p \times 1$ matrix $R_i$ for bucket $i$ and then averaging across the $l$ buckets; $R_i$ is a  binary vector with exactly one $1$ at a random coordinate, so when $R_i$ is when multiplied by a (scalar) input $x$ it puts it into a random coordinate of the output and keeps others zero. Now if we take $l \ge \tilde \Omega(s_p)$ such $R_i$ matrices with high probability, each of the $s_p$ coordinates will be $1$ in some of the $R_i$s. Thus in terms of representation, one can store the specific co-ordinate of $y$ in all the $x_i$ (scaled by $l$) where $R_i$ has a $1$ in that coordinate. Now the expected value of the average assembled from $l$ buckets will be $y$ in expectation -- high concentration can be achieved by making $l$ sufficiently large. Thus a specific co-ordinate of $y$ is stored in $1/s_b$ fraction of the buckets. It can also be ensured that this happens during training by using regularization; the regularization will force all the values in such buckets to be equal and identical to the desired value of that coordinate in the program. The exact same argument extends to the case when $s_b > 1$ except that now $R_i$ is a $s_p \times s_b$ random matrix where each column has exactly one $1$ in a random position.
\end{proof}

{\bf Programs may modify a global-program}: So far we assumed that all the $n$ tasks are independent. However instead we could have a global-program so that all variants of that global task that is already available. Note that the main algorithm loop states that a program in a bucket is initialized from a program in the nearest non-empty context bucket. If we assume that the global-program is in a bucket that is nearest to a new bucket then it will automatically start from there. Further note that we may not even need to copy the entire program to the new bucket, but simply train the delta (modification) there; thus in the new bucket we would store a pointer (sketch) to the global program and the delta represented as a vector. This gives the following claim.

\begin{claim}\label{claim:global-task-appendix}
Claim ~\ref{claim:independent-tasks} holds even if all $n$ tasks are derived from a global task.
\end{claim}








\section{Architecture v1}\label{sec:appendixarchitecturev1}

This version will be used to learn a (latent) DAG of tasks where each task corresponds to the subtree rooted at a node. There is a (learnable) function at each node that recursively takes inputs the outputs of its child nodes. We show how this DAG (or an equivalent) one gets automatically learned in our architecture. The main argument is inductive where we show that the function at each node (or its equivalent) gets programmed at some bucket in our LSH table. The key challenge is in figuring out exactly which tasks are the child tasks for a new task to learned. In the worst case this can be done by trying all possible ${N \choose d}$ subset of $d$ nodes. In practice there may be hints in the input that can used to narrow the search space in to a smaller set of candidates. Section~\ref{sec:v2.decision-tree} shows how this can implemented using a modular decision tree that itself fits well within our architecture.

\begin{lemma}\label{lemma:v1_proof_appendix}
Suppose at each step, a task $t$ is chosen uniformly random from the set of tasks $\{t_1, \dots, t_N\}$ in a DAG of height $\ell$, along with one random sample $(x, y)$ where ${\hat{\phi}}^{(t)}(x) = y$. Then after 
$\ell M N \ln(1/\delta)$ steps all the tasks will be well-trained 
{ (training error rate $\leq \epsilon_L$ for each module at level $L$) }
 w.h.p. We will call SGD $O( \ell M N^{(1+d)} \ln(1/\delta))$ times during the training. Here, $M$ is  is the upper bound of all $M_L$. 
\end{lemma}

\begin{proof}
The learning algorithm follows the framework of Alg. \ref{algo:informal_execution-loop}. 
Let $t'$ be the one of the tasks that $t$ depends. Then we have that
\begin{equation}\label{prob_t_appears}
Pr[t \text{ does not appear } \tau \text{ steps after } t' 
]
= (1 - 1/N)^\tau \leq e^{-\tau/N}.    
\end{equation}
Therefore, after $t'$ is well-trained, if we wait for at least $\tau = N \ln(1/\delta)$, with probability $1-\delta$, $t$ will appear. Without loss of generality, we can assume that $t'$ is the the last sub-module task of $t$ that gets well-trained. Then after $\tau$ steps the training for $t$ becomes useful because we can call these well-trained sub-modules. Note that the probability in Eq.~\eqref{prob_t_appears} applies to any time step, so after the first $t$ arrives, if we wait for another $\tau$ steps, $t$ will appear again. Suppose $M$ is the amount of data that is needed to train function $m \in \mathcal{M}$, then after $M \tau = M N \ln(1/\delta)$ steps $t$ can be well-trained w.h.p. Since we know that basic tasks can be trained without calling other sub-modules, by using standard induction argument we know that all the tasks can be trained within $\ell M \tau = \ell M N \ln(1/\delta)$ steps.
(If $M$ is larger than $\log(N)$, then we only need $O(MN) $)
Because each bucket will maintain at most $O(N^d)$ models at a given time and will run one pass of SGD of each of them upon receiving a sample, we will call SGD for at most $O(\ell M N^{(+d)} \ln(1/\delta))$ times. 
\end{proof}

\begin{remark}
Note that we don't need to pass each data to all the $O(N^d)$ buckets at the same time. We can randomly choose buckets. For example, if $d = 10$ but the compound module  only calls $2$ submodules, then with high probability, we only need to run $O(N^2)$ steps. Further in practice the exploring among all $N$ tasks may not be needed as there may be some smaller candidate subset of only related tasks that need to be considered,
\end{remark}

In the proof of Lemma \ref{lemma:v1_proof_appendix} we implicitly assumed that all the different combinations of child tasks are tried in a single bucket for the parent task indexed by $h(s_t)$. However, in fact, there is limited space per bucket and the different combinations are actually tried in different contexts and hash buckets.
The following claims provide details about exactly which buckets are used in the training of a new task $t$.

\begin{claim}\label{claim:can-learn-parent-task}
Assuming child tasks are learned, the parent task will be learned in some bucket of the table (not necessarily the bucket corresponding to its original task-description context) in a further $O(nM/p)$ steps, where $p \in (0,1)$ is the probability of the routing module choosing the correct subset of children for the task.
\end{claim}
\begin{proof}
Suppose $C$ is the task id context for a task whose child tasks have all been learned. By our assumptions on task context similarity, the buckets corresponding to the child tasks will be among those that the routing module finds when it looks for trained buckets near to $C$. Therefore when the routing module runs the $k$ nearest buckets, combines their results, and chooses some subset of the components to keep as the context of the resulting sketch, it keeps exactly the right components in order to successfully learn the task with some probability $p$.

The context $C'$ of this new sketch references both the original task context $C$ and the combination of previous modules that contributed to it, so there is a separate hash bucket for each possible combination that the model tries, which prevents catastrophic forgetting while the routing module searches for the best combination. After processing at most $\lambda + nM/p$ examples (where $\lambda$ represents how many examples were processed before the prerequisite modules had matured) the function in bucket $h(C')$ will have with high probability learned to perform the parent task. 
\end{proof}

\begin{claim}\label{claim:functioncall-appendix}
Assuming the learning of a task has happened as per Claim~\ref{claim:can-learn-parent-task}, over time the execution pathway for a node gets programmed into the original bucket $h(s_t)$ for that task.
\end{claim}

This follows from the knowledge graph principle, i.e. that outgoing edges point to commonly co-occurring sketches. Intuitively, it corresponds to how a human learns to perform a frequently-performed task so well over time that they don't have to think about the individual steps, it just happens ``automatically''.

\begin{proof}
Let $C$ be the context for a task, and suppose that the model has learned to perform this task by calling some other modules with contexts $C_1, ..., C_r$ and then acting on the compound output of these in bucket $h(C')$. 

Every time $h(C')$ performs successfully on an example (e.g. low loss, high reward, etc; however ``success'' is measured in the model implementation), a copy of the sketch is recorded as a high reward co-occurring example for all of the modules in the execution pathway. Many such examples will be ``averaged'' together over time, smoothing away the details of individual examples and highlighting the parts that remain constant, in particular the execution pathway $C \rightarrow \{C_1, ..., C_r\} \rightarrow C'$ -- note that from the compound sketch of a large number of sketches the average value can be recovered (see Claim 5 in \cite{ghazi2019recursive}). This may be one co-occurring example among many for the intermediate modules $C_1, ..., C_r$, but it will dominate the outgoing edges of the knowledge graph at the original bucket $h(C)$ and thus become the program for $h(C)$.
\end{proof}

\section{Architecture v2}\label{sec:appendixarchitecturev2}

Now in v2, unlike in v1, the precise task identifiers are not given explicitly in the input. consider for example a dog whose current task is to ``Listen to masters command and follow that" -- in this case the precise task will depend on what the masters command is; if it is ``fetch ball" then there is a specific module to do that; there may be several atomic modules possibly one per command that may be needed to to this entire task. 

For example the entire architecture can be thought of as one giant compound module defined by some ``boot" module (think of this as the initial module where all input sketches are sent); this boot module iteratively calls the routing module and modules in the hash buckets where derived sketches get routed -- note that iterative calls can also be implemented as a tail recursion where the first iteration recursively hands off the processing for the remaining iterations.

An implicit precise task is a logically coherent subset of training examples from the external world, but the precise task description may not be explicit in the input, but may be identified after routing the input through a sub-dag of modules. Each task maps to a unique context which is determined by applying the $f$ function on the sketch at some level.


Our learning algorithm uses a combination of deep learned individual modules and probabilistic algorithm to connect up these modules.

Here are the exact formulations of the task sets for the dog command execution and the multi digit number recognition examples.

\textbf{Task set example 1:}
\begin{itemize}
    \item task1: input: \{[TASK,``identify command"], [VIDEO,\la video\ra]\} output: [OUTPUT, \la command-word-from-audio-in-video\ra]
    \item task2: input: \{[TASK,``identify command point to relevant object", [VIDEO,\la video \ra]\} output: [OUTPUT, \la position of object of interest in video based on command\ra]
\end{itemize}

Internal implicit modules:
command task i: input: [``execute given command", i,  \la video\ra] output: [\la position of object of interest in video based on command i\ra]

Note here that even though we have some vague task-descriptions, the actual task-id is obtained by running task1. To solve task2 the architecture needs to first have a trained module for task1, figure out that task2 depends on task1, and further that its output is meant to be the true context/task-id for executing task2.

Note about distribution shift: Note that the module 1 here may be trained on some words. Once trained on a few words, it be automatically become usable for new words even though there is a distribution shift.

\textbf{Task set example 2: }
5 digit recognition: input 5 digit image, output the value; builds upon two modules: an image segmentor that produces 5 smaller images, a 1 digit recognizer that takes a smaller image and outputs one digit.
\begin{itemize}
    \item  task1: input: \{[TASK,``1-digit-recognizer"], [IMAGE,\la image-of-1-digit \ra]\} output: [OUTPUT,\la number-0-to-9\ra]
    \item  task2: input: \{[TASK,``5-digit-recognizer"], [IMAGE,\la image-of-5-digit \ra]\} output: [OUTPUT,\la number\ra]
    \item  task3: input: \{[TASK,``5-digit-image-segmentation"], [IMAGE,\la image-of-5-digit \ra]\} output: [OUTPUT,list of five [IMAGE,1-digit \la image \ra]]
\end{itemize}

The following corollary follows from Claim \ref{claim:decision-tree-appendix} except that at the leaf nodes instead of directly getting the reward we have an atomic module being trained at each leaf and the rewards propagate up the tree as the atomic module converges to the right function to receive external rewards for correct predictions. Since $M$ examples are needed to train each atomic module at the leaf, the number of steps get multiplied by factor $M$.

\begin{corollary}
In any task if the probability of picking the right sequence of decisions for perform the task is $p$ and it takes $M$ examples to train the task, then the task can be learned in $O(M/p)$ steps assuming all previous task it is dependent on are already trained. Any future calls to the decision tree will now use this recorded best path.
\end{corollary}

\begin{remark}
Note that different subtrees in the decision tree for the function  $f$  may be trained over time for different tasks. The vague task descriptor $s_t$ is just one of the fields in the sketch (initial one).   For a given task we are only focused on training a specific subtree; however, the entire decision tree for the entire function $f$  is constantly evolving as more and more tasks get trained.
\end{remark}


The following is the main inductive Lemma to prove Theorem \ref{Architecture:v15_main_theorem}

\begin{lemma}[Inductive lemma]
In any new task $t$ with task descriptor $s_t$ that build upon previously existing tasks that have already been learned to perform well. 
By induction the probability of picking the right sequence of decisions for perform the new task is $p = 1/  2^{O(d^2 +  d \log(N/d))}$ (including which identifying which previous possibly implicit tasks it depends on and wiring them correctly with the right contexts) and it takes $M$ examples to train the task, then the task can be learned in $O(M 2^{O(d^2 +  d \log(N/d))}  )$ examples for each of the $g_i$ atomic modules assuming we have already learned to perform all previous task it is dependent on. 
\end{lemma}

\begin{proof}
The learning algorithm follows the framework of Alg.\ref{algo:informal_execution-loop}. 
The circuit routing is also done by Alg.~\ref{algo:fdecisiontree}: we feed all the $O( {N \choose d} 3^{d \choose 2} )$ candidate edges of the circuit to Alg.~\ref{algo:fdecisiontree}, which finds the correct subset. 
The inductive guarantee that lower-level tasks are well-trained comes from the bottom-up online algorithm of v1. 
Modules 
are marked as mature based on performance, and new modules are only built on top of mature previous nodes.
The probability of picking the right sequence of decisions for perform the new task is $p = 1/  2^{O(d^2 +  d \log(N/d))}$ (including which identifying which previous possibly implicit tasks it depends on and wiring them correctly with the right contexts) and it takes $M$ examples to train the task, then the task can be learned in $O(M 2^{O(d^2 +  d \log(N/d))}  )$ steps per atomic module.
\end{proof}

Theorem \ref{Architecture:v15_main_theorem} follows by applying the previous lemma inductively. We assume for simplicity that all example are uniformly distributed across the total of $G$ atomic modules. So only $1/G$ fraction of examples will be destined for a given atomic module giving a factor $G$ multiplier; the additional $\ell$ multiplier comes from the $\ell$ levels of hierarchy the dependency DAG.

We now formally state that the two task examples can be learned.

\begin{corollary}\label{thm:v2examples}
\textbf{Task set example 1} and \textbf{Task set example 2} can be learned by our architecture if training data for different tasks are input in random order. This follows from previous lemmas.
Given training examples for different tasks in random order,
including for this combined task our architecture automatically learns to use the  output of one of the tasks as a context and builds a downstream module for each context value.
\end{corollary}
\begin{proof}
Although this follows from Theorem~\ref{Architecture:v15_main_theorem}, for illustration we show the proof specifically for these examples to show the exact sequence of events of how this is accomplished. We will argue for example 1 and the second example is similar: note that there are two external task descriptions. So the routing module will send these examples to two hash buckets based on the external task IDs. 
So two modules atomic modules will get trained for each of these tasks at two different buckets. However, only the first task will get trained successfully to a good accuracy (if the second task also gets trained successfully then we are done). 
Now for the second task there is an option to build a compound module which will call the first task.  Now the routing module will use the decision tree to explore different ways of building a compound module for the second task. The right combination involves the following: decide that task2 is not atomic,  run task1 on the input, take the output of task1 and only make that as a context, go to hash bucket based on this context and train an atomic module there. Note that since only these three specific decisions lead to success, the initial probability of picking this path is $c^3$ for some constant $c$. Thus after a constant number of possible path ways with separate atomic modules will need to be trained in parallel before we find a successful pathway. So $O(M / c^3 )$ additional training steps should suffice to train task2 after task1 is complete.  Once the right sequence of calls has been established, this could be programmed as a compound module in the bucket for task2. 
\end{proof}

\begin{remark}
In these examples we simply extract the external task ID which is the first field of the input and use that as the context for the next iteration. However in general this may be a very complicated process. This extraction of the task ID (even what we call as the external task id here) may itself be an evolving compound module consisting of a combination of different atomic  modules and evolving $f$ function decision trees branches  over time.
\end{remark}

\section{Using the Knowledge graph}\label{app:knowledgegraph}

In the following we will assume that there is social network of constant degree and we see images of pairs of people chosen at random from this network.

\textbf{Knowledge graph Task example 0: }
\begin{itemize}
    \item Task1: [“remember sketch”, [IMAGE,\la image of pair person1 and person2 \ra]

“remember event”  task is an unsupervised task only meant to record the sketch once the count of its context has exceeded a certain threshold and is not meant to predict any kind of output.
We will assume there is a person recognition module that takes the image as input and outputs a compound sketch of two person sketches for the persons in the image (later we will see how the routing module can automatically learn to route the input to such a module without assuming it).
\end{itemize}

The following Theorem is a consequence of the "Knowledge graph" principle that is implemented in line 27 in Algorithm\ref{algo:execution-loop-full}. 

\begin{claim}\label{claim:latent-knowledge-graph-appendix}
Suppose we have a module that has learned to identify faces from images and return the identity of those people. Given input data of images of pairs of people, where the pairs are chosen from a uniform distribution given by edges of a graph, the knowledge graph created by our architecture contains a subgraph homeomorphic to this original graph.
\end{claim}

\begin{proof} Given an input sketch [IMAGE, \la bit-map \ra], by Claim~\ref{claim:independent-tasks} it gets routed to a person-recognition module. That returns a compound sketch of the set of all people in the image, so it will return the set \{ \la person-1-sketch \ra, \la person-2-sketch \ra \}. This compound sketch will go to a new bucket, which will get pointers to the original \la person-1-sketch \ra and \la person-2-sketch \ra buckets due to co-occurrence. See Architecture Principle (4) in section~\ref{appendix:architecture}. If we take the subgraph of the knowledge graph consisting of all pairs of person sketches and pointers to individual person sketches, this will be homeomorphic to the original graph. 
\end{proof}

\textbf{Knowledge graph Task set example 1: }
\begin{itemize}
    \item Task1: [“remember sketch”, [\la person1 \ra, \la  person2 \ra]]

“remember event”  task is an unsupervised task only meant to record the sketch once the count of its context has exceeded a certain threshold and is not meant to predict any kind of output.

We assume that the sketch of a person is stable or resistant across different instances of a person sketch.

    \item  task2: [“Find common friends”,  [\la  person1 \ra, \la  person2 \ra]], Output: [\la list of common friends \ra]

\end{itemize}

\textbf{Knowledge graph Task set example 2: }
\begin{itemize}
    \item Task1: [“remember sketch”, [IMAGE,\la image of pair person1 and person2 \ra]

    \item Task2: [“remember sketch”, \{ [IMAGE,\la image of person \ra], [NAME, \la name of person] \}

    \item Task3: [``Extract list of persons (as features) from image", [IMAGE, \la image containing multiple people \ra], Output: [\la list of person features from image \ra]

We assume task 3 can be solved in a way where it extracts stable person-features from images that result in same ``fingerprint" for the same person possibly appearing across images.

    \item Task4: [“Find common friend names”,  [[NAME, \la name of person1 \ra, [NAME, \la name of person2 \ra]], Output: [\la list names of common friends \ra]

\end{itemize}

To demonstrate that knowledge graph is a useful extension, 
we first note that, example 1 cannot be learned with simple modules without knowledge graph. 
\begin{claim}
Without using knowledge graph memory, training a neural network submodule for task2 in Example 2 can only achive accuracy at most $O(\sqrt{n/N})$, where $n$ is the total number of bits used to storee all the weight of the neural network and $N$ is the number of people in the data. 
\end{claim}
\begin{proof}
This follows from a similar argument based on mutual information as in \cite{sketchmem2021}.
W.L.O.G. we can assume that the number of common friends is $1$ for the sake of the lower bound proof. 
\end{proof}

\begin{claim}
All tasks in Example 1 can be solved jointly from training data in polynomial time. This can be solved using the knowledge graph edges.
\end{claim}
\begin{proof}
Task1 does not involve any prediction. For the second task the architecture will first try to train an atomic module but will fail.  overtime because of task1,  a knowledge graph of friend connections will you get created between the sketches \la person1 \ra, \la person2 \ra  and the compound sketch  [\la  person1 \ra, \la  person2 \ra ] (based on architecture principle 4, line 7 in the pseudo code). After this in the first iteration of the architecture the knowledge graph edges would be an extracted (line 15 in the pseudo code) for the input sketches \la  person1 \ra , \la  person2 \ra .  these edges will point to the list of all friends for \la person1 \ra  and \la person2 \ra  respectively.  In the second iteration of the architecture with some probability it will make the set of these two list as a compound sketch for the next round and [“find common friends”, \la extract first part-edges \ra , \la extract second part-edges \ra , \la take-combination \ra ]  as the new context,  and we'll start training  an atomic module at this round. Since finding the intersection of two lists is a simple task this training will succeed to make the correct prediction. Overtime this pathway (routing module decision tree choices)  of extracting neighbors of \la person1 \ra   and \la person2 \ra ,  making a set out of the two lists,  and giving it to that new atomic model will get  strengthened, and eventually hard-coded in the original bucket for the “Find common friends” task.
\end{proof} 

Note: We remember sketches that occur more than a certain fraction of time. If there are $m$ buckets we track events that occur more than $1/m$ fraction of the time -- this ensures that there is space for all frequently occurring contexts.

\begin{claim}
All tasks Example 2 can be solved jointly from training data, given inputs from Example 2, we can learn all tasks in polynomial time (this can be solved using the knowledge graph edges).
\end{claim}
\begin{proof}
This is merely a generalization of the earlier proofs but goes through higher number of iterations of algorithm 1 along the lines of the proof of Theorem~\ref{thm:v2examples}. Task 3 is a leaf level task that is used by Tasks 1 and 2. Task 2 needs to route the image part of the input sketch to Task 3 getting \la person-features\ra for the person in the image; then $f$ function needs to create the compound sketch [\la person-features \ra, \la name of person \ra] as the context. Then this context is remembered in the hash buckets including its component sketches that point to the compound sketch and vice versa. Since this involves making a constant number of correct routing decisions, this will happen with constant probability. Also, Task 1 needs to route the image in its sketch to Task 3 to get the person-features for the two people in the image and then the pair of person features needs to be remembered in their bucket with appropriate bi-directional pointers to the individual person-features. We now have all the edges between pairs of friends and between a person and their name. with these established, Task 4 needs to use the names of the two people to lookup the edges to find their person-features and then just like in example 1, use friend edges to find two lists of friends for each of the persons, and then convert these two lists to two lists of names and then train a final atomic module to find the intersection of these lists of names. Assuming a constant degree friendship graph, all these choices will line up with at least constant probability. To see why all this is a constant number of probabilistic choices, think of the recursive view of Algorithm \ref{algo:execution-loop-full} mentioned in the end of section \ref{sec:appendixarchitectureprinciples}. All that is needed is that in a certain context (that depends on the content of the recursive call stack) the routing module when given a person-sketch as input can be probabilistically trained to return a name for that person sketch, and then in some other context also return a list of names for a list of person-sketches, and then again in some context to convert two lists of person sketches to two lists of name sketches; these are all a sequence of decision choices in a combined view of a decision tree for the routing module; over time the correct probabilistic choices get strengthened based on external rewards to arrive at the right decision tree and atomic module for Task 4.
\end{proof}

\section{Architecture v3: Q-learning and other Reinforcement Learning tasks}\label{sec:architecture-v3}

In this section we will show how Reinforcement Learning (RL) tasks may be solved by our architecture by executing algorithms such as Q-learning. 
The main results of this section are that an extension of our architecture, Architecture v3, can perform tabular Q-learning so that it can learn to solve multiple RL problems at once without confusion (Claim~\ref{claim:q-learning}), and that it can work out how to use other modules (e.g.\ image classification) to improve its policy-learning, thus producing a form of ``modular RL'' (Theorem~\ref{thm:modular-rl}).

\subsection{Lifelong reinforcement learning}

Several different formal RL tasks are studied in the literature. Here, we focus on episodic RL problems:

\begin{definition}
Episodic RL problems are defined as follows: We fix a set of {\em states} $S$ and {\em actions} $A$. An {\em environment} is given by a Markov Decision Process $P$, that for a given pair $(s,a)\in S\times A$, specifies a distribution over new states $s'\in S$; there is also a distribution $\rho$ over starting states, and a {\em reward distribution} $R$ that for each pair $(s,a)\in S\times A$ gives a distribution over real-valued {\em rewards}. There are $T$ episodes of length $H$ each. In an episode, a starting state is drawn from $\rho$ and revealed to the agent. Then, for $H$ steps, the agent is allowed to choose an action $a\in A$, the environment transitions to a new state according to $P$ and gives the agent a reward according to $R$.  
\end{definition}

For simplicity, we will first consider a lifelong learning setting in which the episodes for different, independent environments are interleaved, similarly to the setting of Section~\ref{sec:v0}.

\begin{definition}\label{def:multi_RL}
We define the lifelong RL problem with independent environments as follows: suppose we have $N$ environments with their own corresponding MDPs and rewards. The interaction still consists of episodes of length $H$. At the beginning of an episode, one of the $N$ environments $E_i$ is chosen by $i \in \texttt{Uniform}([N])$, and a starting state is generated $s \sim \rho_i$. $i$ and $s$ are revealed to us, and we interact with $E_i$ for $H$ steps.
In the next episode, a new environment is again chosen by independently sampling $i \in \texttt{Uniform}([N])$ and a starting state is independently sampled $s \sim \rho_i$, and we interact with the new environment for $H$ steps. We keep doing this for $T$ episodes.
\end{definition}

We will assume, moreover, that the interaction with the environments during an episode of the lifelong RL problem has a specific form, as follows.
We introduce new, specific types of input and output sketches: one to input a state from the environment, another to output an action, and a third to possibly receive a reward for that action.
The input data for RL problems arrives as sketches of state input from the environment and a possible set of actions in the form $S_{in}$ = [RL-CONTEXT,[\la rl-state\ra, \la possible-actions\ra]]. An action is taken by outputting an [ACTION, \la action-choice\ra, \la rl-state\ra] sketch for choosing a specific action in the state rl-state.  Rewards for an action are provided as the compound tuple sketch [[REWARD, \la r\ra], [ACTION, \la taken-action\ra, \la rl-state\ra]]. (Note that our convention is that the all-caps fields here represent some constant label/enum type and lower case fields may be variable ``arguments"). [[Nte that multiple RL problem instances can be fed into our system as the state could correspond to the state from any of the problems. Later in subsection .. we will see how the state may not be given explicitly but may need to be inferred using other modules just as task and context is inferred in v2]]

{\bf Architecture v3 details}: To handle such RL specific sketch inputs and outputs, architecture v3 extends Architecture v2 by introducing PROGRAM-type sketches and EXECUTE-type sketches, allowing it to pass programs to other buckets and execute these programs in the new buckets. The former has the format [PROGRAM, \la program-sketch\ra], where \la program-sketch\ra can be interpreted as a program that can be executed on some input. When the routing module comes across a sketch 
s = [[EXECUTE, [PROGRAM, \la program-sketch\ra]], \la input-sketch\ra], it executes the program \la program-sketch\ra on input \la input-sketch\ra after going bucket corresponding to $f(s)$ -- no separate program needs to be created at $h(f(s))$.

\subsection{ Tabular Q-learning}\label{app:qlearning}

In this section we show that our architecture can be used to perform Q-learning.
Specifically, we will consider Q-learning with a tabular episodic Markov Decision Process (MDP). 
In this environment, there is a set of states $\mathcal{S}$ with $|\mathcal{S}| = S$, and a set of actions $\mathcal{A}$ with $|\mathcal{A}| = A$. Each episode has a finite number of steps $H$. 
At the beginning of each episode the initial state $x_1 \in \mathcal{S}$ is arbitrarily chosen. At each step $h$, the agent observes a state $x_h \in \mathcal{S}$, picks an action $a_h \in \mathcal{A}$, receives an reward and transitions to the next state $x_{h+1}$ drawn from some distribution $P_h$ depending on the current state and action.
In this MDP, we use $P_h$ to denote the transition probabilities at step $h \in [H]$: $P_h(\cdot | x, a)$ is the distribution over next state that the agent got transitioned into if  at step $h$ it takes action $a \in \mathcal{A}$ at state $x \in \mathcal{S}$. Let $r_h(x, a) \in [0, 1]$ denote the deterministic reward it receives. 
We use $\pi_h$ to denote the policy of actions at $h$. That is, $\pi_h: \mathcal{S} \to \mathcal{A}$. 
The goal is to compute policies such that the actions chosen by the policy at each step give accumulated rewards at each episode is maximized.

First  note that the decision tree learning algorithm in section~\ref{sec:v2.decision-tree} can be viewed as a special case where the state action graph is a tree and all rewards are at the leaves. We extend that idea to Q-learning with general state graphs.
%
%
Now consider a general state-action graph, not just a decision tree. We give here a high-level overview of how an RL algorithm can be implemented in Architecture v3, with a view to showing that it can implement a tabular Q-learning algorithm (Claim~\ref{claim:q-learning})

We define a special module called the RL-module, that gets executed on sketches with the context RL-CONTEXT. This module outputs a sketch $S_{RL}$=[PROGRAM, \la rl-state-sketch \ra] -- note that it doesn't execute the RL algorithm but simply outputs it as a program. We assume that this program sketch is hard-coded into the architecture, since the goal here is to show that our architecture can learn using a specified RL algorithm, not that it is capable of developing its own algorithm from scratch. 

When $S_{in} = $ [RL-CONTEXT,[\la rl-state\ra, \la possible-actions\ra]] is input, the RL-CONTEXT context is looked up to get $S_{RL}$, and then $S_{in}$ and $S_{RL}$ are combined using attention to produce the compound {\em state sketch} $S$ = [[EXECUTE, $S_{RL}$], $S_{in}$]. This then goes to the bucket $h(f(S))$ where \la rl-state-sketch \ra is executed on input [rl-state, possible-actions] to output a specific taken-action sketch $S_{action}$. The $S_{action}$ is a recursive sketch [ACTION, \la taken-action\ra, \la rl-state\ra] that leads to (or is followed by) the reward input sketch $S_{reward} = $ [[REWARD, \la r\ra], [ACTION, \la taken-action \ra, \la rl-state \ra]] for that taken-action edge and a next $S_{in}'$  sketch that inputs the next state from the environment. Since this will get propagated back along the knowledge graph, the reward will get accounted at the bucket for [ACTION, \la taken-action\ra , \la rl-state\ra] (or equivalently, the outgoing edge of rl-state corresponding to taken-action). The rl-state-sketch encodes the specific details of the RL algorithm, e.g. hyperparameters and exploration method, tracking rewards on each action, tracking temperature, and converging on the best action for a state. 



Note also that the rl-state and action may be discrete states/actions or sketches of more complex/continuous states and action possibilities -- in the latter case we are taking advantage of the ``discretizing'' property of the $h(f())$ function that maps sketches to hash buckets-ids to simplify our state/action space.

\subsubsection{One loop of the Q-learning algorithm}

\begin{enumerate}
\item Input arrives as a sketch $S_{in} = $ [RL-CONTEXT,[\la rl-state\ra, \la possible-actions\ra]]  containing the state $S = $ \la rl-state\ra, a list of  \la possible-actions\ra $ = [ A_i ]$, and the hint RL-CONTEXT that this a Q-learning problem.
\item Because of the RL-CONTEXT context, this gets sent to the Q-learning bucket.
\item The Q-learning bucket outputs $(P,S)$, where $P$ is a program.  get $S_{RL}$, and then $S_{in}$ and $S_{RL}$ are combined using attention to produce the compound {\em state sketch} $S$ = [[EXECUTE, $S_{RL}$], $S_{in}$].
\item In the bucket $h(f(P,S))$, we run $P$ on $S$. $P$ looks at the list of actions / outgoing edges and samples an action $A$.
\item The output of the bucket $h(f(P,S))$ is a command to take action A in the environment. This generates a new state $S'$ and reward $R(S', A, S)$. Sketch these into a new sketch $S'$ and Q-learning hint.
\item As above, $f$ sends $S'$ to the Q-learning bucket to pick up the program $P$ and outputs $(P, S')$.
\item In the bucket $h(f(P,S'))$, the program $P$ looks at the outgoing edges to get the Q-values $Q(S',A')$ and computes the new Q-value for $(S,A)$.
\item Backprop: update the outgoing edges and Q-values.
\end{enumerate}


Note: if the number of actions for each state becomes large, then an alternative version where we visit the state-action buckets may work better.

\begin{claim}\label{claim:q-learning}
Architecture v3 can do tabular Q-learning. Its implementation is compatible with the UCB-Hoeffding algorithm given in \cite{jin2018q}, ensuring it can learn an $\epsilon$-optimal policy in $O(1/\epsilon^2)$ episodes. Further, Architecture v3 can solve multiple RL problems at once, without conflict between the different sets of Q-values and without needing to know in advance how many separate problems there are or allocate resources in advance.
\end{claim}



\begin{proof}
We assume that the input data is formatted as described above, and that the RL-module with hardcoded Q-learning program-sketch is provided. The Q-learning algorithm needs to do two things: at state $s$ it needs to choose an action $a$ according to some exploration method (random, greedy, $\epsilon$-greedy, etc), and at the subsequent state $s'$ it needs to identify the maximum Q-value for $s'$ and perform the tabular Bellman update 
\[Q(s,a) \leftarrow (1-\alpha)Q(s,a) + \alpha\Big(R(s',a,s) + \gamma \max_{a'}Q(s',a')\Big)\]
for $s$; we can assume that hyperparameters such as the learning rate $\alpha$ and the discount factor $\gamma$ are both encoded in the RL-module. We use the recursive nature of sketches to combine these two steps: essentially we describe the process at state $s'$ and note that the ``update previous state'' step is empty if the rl-state sketch $S'$ does not point to a previous state.

So let $S'$ be the current rl-state-sketch, which (if it occurred as a result of taking some action $a$ = \la previous-taken-action \ra at rl-state $s$) includes a recursive copy of the sketch [\la previous-taken-action \ra, $S$] and also the value of the reward $R(s',a,s)$ obtained from this action.

At the bucket $h(f(S'))$ the Q-learning program-sketch looks up the maximum Q-value of the available actions at $S'$ (these may be stored e.g. as a table in the bucket $h(f(S'))$ or as weights on the outgoing edges) and uses this to both choose its next action and to compute the reward $R:=R(s',a,s) + \gamma \max_{a'}Q(s',a')$ to be accounted to the previous state-action pair (if any). It returns this information as the compound tuple sketch [[ACTION, \la taken-action \ra], [REWARD, $R$]]; recall that \la taken-action \ra is a recursive sketch that includes the previous state-action pair (if it exists). Now this action is executed in the environment (which may provide another state $S''$ for the next round), while the knowledge-graph-updating process passes this sketch back to the bucket $h(f(S))$ of the previous state-sketch. Its similarity to the edge with matching \la previous-taken-action \ra component ensures that the reward is accounted to the correct state-action pair, and since the sketch compounding process essentially produces a weighted average of similar components (with weights that can be specified by the user/RL-module), this completes the Bellman update.

To see that this is compatible with \cite{jin2018q}, observe that we need only change details that are hardcoded in the RL-module: replace $\alpha$ with $\alpha_t:= (H+1)/(H+t)$ where $H$ is the episode horizon and $t$ is the frequency count for this state-action pair, and add a bias term $b_t := c\sqrt{H^3\iota/t}$ to the reward ($c$ and $\iota$ are constants given in \cite[Theorem 1]{jin2018q}). Based on online-to-batch conversions, the regret can be arbitrarily small when the number of episodes is large enough; thus we can achieve $\epsilon$ error rate if we have $O(S A / \epsilon^2)$ episode samples.
\end{proof}

\begin{claim}
The above implementation can be extended to include deep Q-learning, where each state bucket learns and stores a parametrized Q-function for that state.
\end{claim}

\begin{claim}\label{claim:multiple-rl-problems-appendix}
Architecture v3 can solve multiple independent RL problems at once, without conflict between the different sets of Q-values and without needing to know in advance how many separate problems there are or allocate resources in advance.
\end{claim}
\begin{proof}
New hash buckets are created by the architecture as new contexts (i.e. rl-states) arise, allowing it to expand dynamically as needed. The Q-values for each state are stored locally within the corresponding rl-state bucket.
\end{proof}

{\bf Advantages of using Architecture v3 for Q-learning}: Implementing Q-learning in our architecture comes with several key benefits, including:
\begin{itemize}
\item Graceful generalization to continuous state spaces: since $f$ drops the extraneous environmental details from a sketch, the model automatically groups together similar states.
\item Learning an environment model: By passing sketches of executed actions back along the knowledge graph, the outgoing edges of the bucket corresponding to a state $s$ can store not only the Q-values of state-action pairs $(s,a)$ but also frequency counts of tuples $(s',a,s)$: that is, we learn a model of the transition function as a free side-effect.
\end{itemize}

\begin{remark}
While we have focused on tabular Q-learning here for simplicity, we note that many different RL algorithms could be ``dropped in'' simply by changing the program-sketch provided to the RL-module. We have also restricted our attention here to what happens one step back along the knowledge graph, but by backtracking the sketches further it becomes possible for earlier states to use this data in future to ``look ahead'' several steps. Finally, we note that this modular architecture should also lend itself well to deep Q-learning approaches, where each state bucket learns and stores a parametrized Q-function for that state; however, the implementation of this is beyond the scope of this paper.
\end{remark}


\subsection{Modular Q-learning} 
In the problem of Sec.\ \ref{app:qlearning}, each state can be explicitly observed. However, in reality, the recognition of the states might need to be learned.
The real power in our architecture comes from its ability to seamlessly combine RL decision making with other types of task, e.g. classification. As a simple example, consider a situation where we need to choose one of a limited number of actions in response to an image of a person displaying one of several gestures. A single large RL model probably {\em could} learn an effective policy for this task, but since it can't identify the indirect association {\em image $\rightarrow$ gesture $\rightarrow$ response to gesture} it could just have easily have learned to react to some spurious patterns in the training dataset. Meanwhile, our architecture would simultaneously try this approach (i.e. try to learn a policy directly from the images) and also explore the possibility of using related modules as part of its decision. Assuming it had already developed a gesture-classification module, pathways that make use of this module would be consistently high reward and therefore preferred over the direct approach. Indeed, as we show below, we need not even assume the prior existence of the gesture-classification module: our architecture can learn to solve both problems simultaneously (assuming it is provided with training data for both problems).

\begin{definition}
We define the modular RL problems as follows: suppose we have $N$ environments with their own corresponding MDPs and rewards, and $N'$ classification problems where one identifies the states of the RL. 
At each time step, uniform randomly, we are either given a labeled sample of one of $N'$ classification tasks, or put into one of the $N$ environment and interact with it for $H$ steps as in \textbf{Definition}~\ref{def:multi_RL}.
We assume that the data distribution of state identification task is $P$. 
\end{definition}

\begin{definition}
A distribution $P$ $m$-dominates distribution $Q$ if for all $x$ in the sample space, we have $Q(x)\leq m P(x)$. 
\end{definition}

If the data distribution $P$ of state identification during training $m$-dominates state classification distribution $Q_i$ conditioned on state $s_i$ $ \forall i = 1,\dots, S$, then if classification module can achieve error rate of $\epsilon/m$ on $P$, it can also achieve error rate of $\epsilon$ on each $Q_i$. Indeed, let $E$ be the event that we make an error in classification, then since $P(E) = \sum_{x \in E} P(x)$ we have that for all $i$:
$$
Q_i(E) = \sum_{x \in E} Q_i(x) \leq 
\sum_{x \in E} m P(x) 
= m P(E) \leq \epsilon. 
$$

However, to guarantee $P$ can $m$-dominate \emph{every} $Q_i$, then we need $m\geq S$. Indeed, consider that the supports of all $Q_i$ are disjoint: $D_i \cup D_j = \emptyset \ \ \forall i \neq j$ where $D_i = \{x : Q_i(x) \neq 0\}$. Then we have 
$$
m = m \sum_{\forall x} P(x)
\geq m\sum^S_{i=1} \sum_{x \in D_i} P(x)
\geq \sum_{i=1}^S \sum_{x \in D_i} Q_i(x)
= S.
$$

Intuitively, we are assuming that there is some data-generating distribution $P$ that can cover every $x$ in the support of each $Q_i$ with some properly lower bounded probability mass.   
An example of such a $P$ would be the uniform mixture of $Q_1,\dots, Q_S$. That is $P(x) = (1/S) \sum_{j=1}^SQ_j(x)\  \forall x$. In this case, it's immediate that $P$ can $S-$dominates each $Q_i$ since $S P(x) = \sum_{j=1}^S Q_i(x) \geq Q_{i}(x) \ \  \forall i$. 
From another perspective, if $P$ allows us to visit each state $s_i$ at least $\mu$ fraction of the time, then $P$ can $(1/\mu)$-dominate each $Q_i$.

\begin{theorem}\label{thm:modular-rl}
Given a modular RL problem,
Architecture v3 figures this connection out automatically and uses the classification module as part of its RL solution.
\end{theorem}
\begin{proof}
This follows by combining modules along the same lines as the proof of Theorem~\ref{thm:v2examples}
\end{proof}

This demonstrates one of the key points of our architecture: it is capable of handling multiple types of problem in a uniform way, and hence is able to  combine them and exploit the relationships that arise organically in the knowledge graph. Combined with the fact that it can expand and create new modules as it discovers new concepts, the result is an extremely flexible architecture capable of solving complex multi-layered problems.

\section{Case Study: Card Game}\label{sec:appendixcasestudy}


Consider  a card game where one has to pick a card from the cards in hand that has the same number as on the top card on the deck on table, the desired module is to ``identify card from hand that has same number as that of top card'' and output sketch is ``put that card on the table''. This module gets executed when one hears ``make your next move''.

The main sketches in consideration are [AUDIO, ``make your next move''], [CARD-ON-TABLE, \la top-card-sketch \ra], [CARDS-IN-HAND, \la card-in-hand-sketch \ra], From the on-going game there is also a sketch [CARD-GAME, [NAME, ``sequence''].

Let us assume that the system has already learned about cards and there is a module to recognize the type of a card that takes as input a card-image [CARD-IMAGE, \la image-sketch \ra] and produces a sketch [CARD, \la card-sketch \ra] that contains the number and type of the card. So \la card-sketch \ra = [CARD-DESCRIPTION, \{[CARD-NUMBER, \la number-sketch \ra], [CARD-COLOR, \la color-sketch \ra], [CARD-SYMBOL, \la card-symbol-sketch \ra]\}]

We will also assume based on knowledge of previous card games that the sketches [CARD-ON-TABLE, \la top-card-sketch \ra], [CARDS-IN-HAND, \la card-in-hand-sketch \ra] are produced from the visual input [IMAGE, \la input-image-of-scene] -- based on the sketch [CARD-GAME, [NAME, ``sequence"]] and the sketches that the visual analysis module outputs attention is paid on sketches related to cards.

The compound sketch ([AUDIO, ``make your next move''], [CARD-ON-TABLE, \la top-card-sketch \ra], [CARDS-IN-HAND, \la card-in-hand-sketch \ra], [CARD-GAME, [NAME, ``sequence'']]) hits a new hash bucket that needs to learn the specific new module. Since this bucket is empty it may be initialized from other buckets for similar buckets; if there is a program for another similar card game it will have the same structure but with a different card-game name. A new program will get trained for this new game starting from that program. That program will call one program to ``identify the best card in hand'' and then another one to ``put that card on the table''. Only the former needs to be modified. The version is simple and can be trained from a few examples: from the card sketches on hand and the card sketch on top of the table it needs to output the one in hand that matches the number of the one on table.

\begin{figure}
\includegraphics[trim = 20 50 0 15, clip, width=\textwidth]{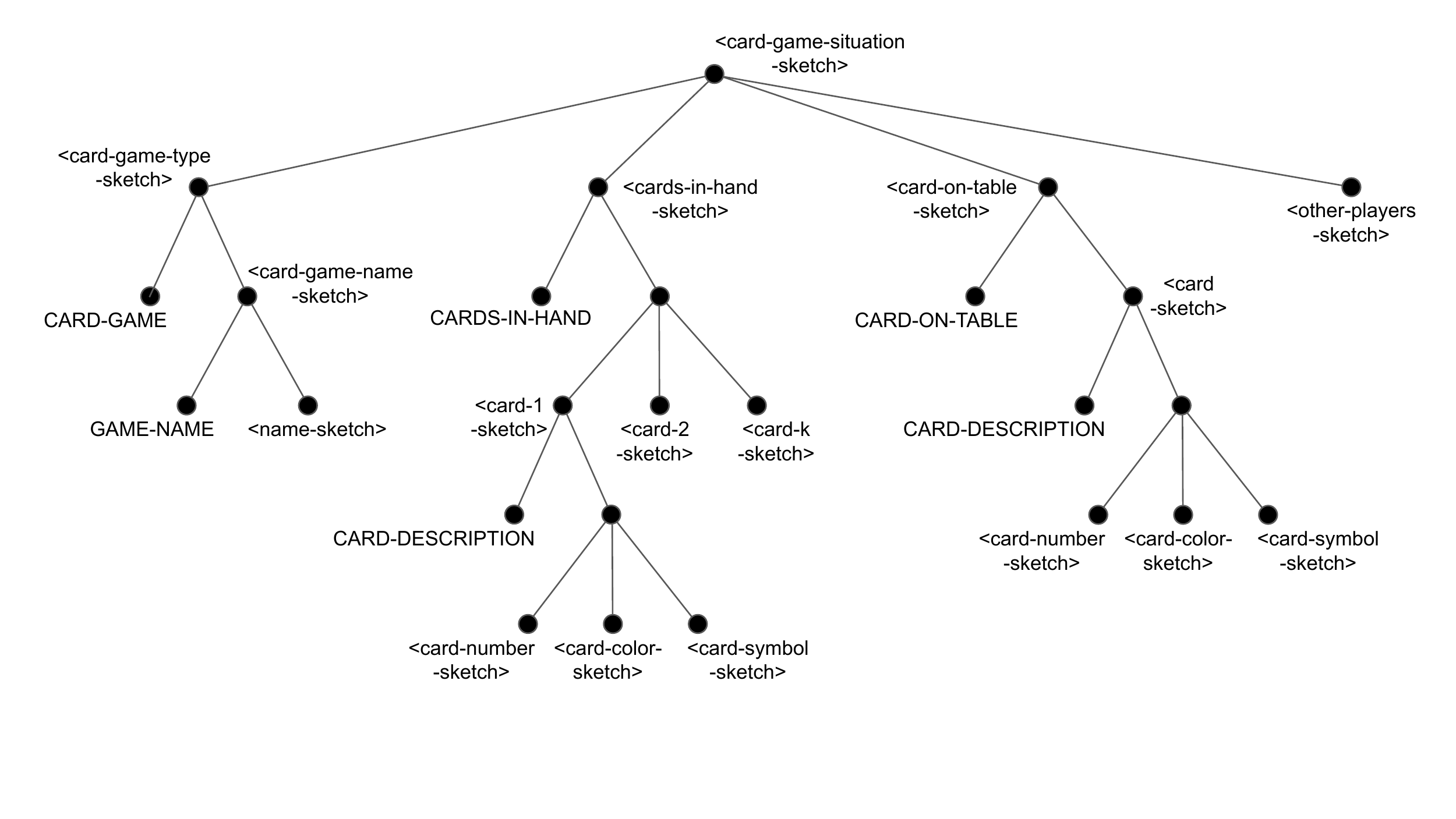}
\caption{Relevant sketches in a card game where one has to pick a card from the cards in hand that has the same number as on the top card on the deck on table. The desired module is to “identify card from hand that has same number as that of top card” and output sketch is “put that card on the table”. The main sketches in consideration are [CARD-ON-TABLE, \la top-card-sketch \ra], [CARDS-IN-HAND, \la card-in-hand-sketch \ra] that get triggered on the input [AUDIO, ``make your next move'']
}
\label{fig:card-game-sketch-tree}
\end{figure}

\section{Misc}\label{app:misc}

\subsection{Correspondence between programs, modules, networks, embeddings}

Each embedding in an embedding table in any deep network can be viewed as a ``module'' that modifies the ``program'' defined by the upper layers. Similarly each embedding in a hash table can be viewed as a ``program'' embedding. Program embeddings can be interpreted more generally where they go through some main/global network that generates an encoding of another deep network.

We could also construct new programs by finding a cluster of related programs and doing a low rank decomposition of their vector representations. The low rank approximation can be viewed as ``subroutines'' and the programs can be viewed as combinations of these subroutines.

\subsection{CNNs as simple recursive functions}

A compound module can ``call'' other modules. With this ability, there is a simple modular recursive view of a CNN instead of the normal bottom up view where we start with small patches and keep combining them to form bigger and bigger patch representations as one goes up the network. 

In the recursive view, the image goes to a module that handles patches at the higher level that recursively calls the module for the lower level patches. Given an input image form a set of patches (of the largest size with the appropriate stride); ``call'' the module for each of these patches to get a sketch for each patch; combine all these sketches for these patches to get a single sketch for the full image. (Note that this is a top down view where the module for the patch of the certain size recursively calls the module for the patches of the next largest size.)

\subsection{Task id's may not be present}

Although we have assumed even in v2 that there are some external Task descriptions (that may not be entirely explicit), in real life we do not get such a separate field but receive only an endless sequence of sights and sounds. The external task description may in fact also be ``inferred" from the video input -- for example the task ``PLAY CARD GAME" may be inferred from the context around the current images or from the the previously seen/heard inputs.

One limitation is that we haven’t gotten into how logic, reasoning, and language could be handled uniformly. While we believe there could be modules that evolve for these, the conceptual details as to how they would work in this architecture have not been investigated.

\subsection{Experiment details}

\subsubsection{Five digit recognition}
For the modular approach, the sub-modules are as follows. 

\begin{enumerate}
    \item For the image segmentation task, the sub-module is a convolutional neural network with the following layers: a convolutional layer with $32$ output channels and $3 \times 3$ kernel; a flatten layer; a fully-connected layer with $128$ output units; a fully-connected layer with $64$ output units, and a fully-connected layer with $5$ output units (output layer, with output being the horizontal segmentation coordinates).
    \item For the single digit recongition task, the submodule is a convolutional neural network with the following layers: a convolutional layer with $32$ output channels and $3 \times 3$ kernel; a flatten layer; a fully-connected layer with $128$ output units; a fully-connected layer with $64$ output units, and a fully-connected layer with $10$ output units (output layer, with output being logits of the $10$ output classes).
\end{enumerate}

For the end-to-end approach, the model is a convolutional neural network with the following layers:  a convolutional layer with $32$ output channels and $3 \times 3$ kernel; a flatten layer; a fully-connected layer with $128$ output units; a fully-connected layer with $64$ output units, a fully-connected layer with $256$ output units, and a fully-connected layer with $10000$ output units (output layer, with output being the horizontal segmentation coordinates).

\subsubsection{Intersection of halspaces}
For the modular approaches, all the modules are $3$-layer fully-connected network. Number of hidden units for each layer is $10$, $50$ and $2$ (output layer). For the end-to-end approach, the model is a $3$-layer fully-connected network with $100$, $500$ and $2$ units.

For the $K=10$ case, we noticed that the final solution doesn't require all the $10$ sub-modules to be trained first. We repeat the experiment for $30$ times and for most times, it trains $7$ sub-modules and the final model uses the output of these $7$ modules and the raw input $x_i$ to successfully predict the intersection of $10$ halfspaces.

\end{document}